\documentclass{article}

\usepackage{microtype}
\usepackage{amsmath, amssymb, amsfonts}
\usepackage{algorithm}
\usepackage{algpseudocode}

\usepackage{graphicx}
\usepackage{subcaption}
\usepackage{booktabs} 

\usepackage{hyperref}
\usepackage{mlmath}

\usepackage[accepted]{icml2026} 

\usepackage{amsmath}
\usepackage{amssymb}
\usepackage{mathtools}
\usepackage{amsthm}

\usepackage[capitalize,noabbrev]{cleveref}

\theoremstyle{plain}
\newtheorem{theorem}{Theorem}[section]

\newtheorem{lemma}[theorem]{Lemma}
\newtheorem{corollary}[theorem]{Corollary}
\theoremstyle{definition}
\newtheorem{definition}[theorem]{Definition}
\newtheorem{assumption}[theorem]{Assumption}
\theoremstyle{remark}
\newtheorem{remark}[theorem]{Remark}

\usepackage[textsize=tiny]{todonotes}

\icmltitlerunning{Gradient Flow Dynamics and Implicit Bias of Diagonal Linear Networks under Infinitesimal Initialization}

\begin{document}

\twocolumn[
  \icmltitle{
  Gradient Flow Dynamics and Implicit Bias of Diagonal Linear Networks under Infinitesimal Initialization
  }



  \icmlsetsymbol{equal}{*}

  \begin{icmlauthorlist}
    \icmlauthor{Jiajie Zhao}{yyy,yyy2}
    \icmlauthor{Jianxing Wang}{yyy,yyy2}
    \icmlauthor{Junjie Yang}{yyy,yyy2}
    \icmlauthor{Zhiwei Bai}{yyy,yyy2}
    \icmlauthor{Yaoyu Zhang}{yyy,yyy2,comp}
    
  \end{icmlauthorlist}

  \icmlaffiliation{yyy}{School of Mathematical Sciences, Shanghai Jiao Tong University}

   \icmlaffiliation{yyy2}{, Institute of Natural Sciences, MOE-LSC, Shanghai Jiao Tong University}
  
  \icmlaffiliation{comp}{School of Artificial Intelligence, Shanghai Jiao Tong University}

  \icmlcorrespondingauthor{Yaoyu Zhang}{zhyy.sjtu@sjtu.edu.cn}

  \icmlkeywords{Diagonal linear network, gradient flow, implicit bias, infinitesimal initialization, saddle-to-saddle dynamics}

  \vskip 0.3in
]




\printAffiliationsAndNotice{}  


\begin{abstract}

We study the gradient flow dynamics of diagonal linear networks for regression tasks under infinitesimal initialization. Extending Theorem 1 from~\citet{pesme2023saddle}, we generalize the analysis to both deep diagonal linear networks and a broader class of two-layer diagonal linear networks (as defined in Definition~\ref{def:general 2-layer}). Specifically, we demonstrate that the training trajectories of these models can be equivalently characterized by the proposed Algorithm~\ref{alg:my_algorithm}. We further prove that this algorithm converges to the solution of a modified \( \ell_1 \) norm minimization problem. As a result, we establish that the implicit bias of both network architectures corresponds to a modified \( \ell_1 \) norm in the regime of infinitesimal initialization. Additionally, we provide insights into the underlying mechanisms governing these dynamics by identifying the Structural Invariant Manifold (SIM)~\cite{zhao2026architecture} as the key geometric structure that shapes the learning process.
\end{abstract}

\section{Introduction}
The remarkable success of deep learning relies  on the ability of overparameterized neural networks to generalize well to unseen data, even when they possess enough capacity to memorize the training set with random labels~\cite{zhang2017understanding}. Classical learning theory, which relies on uniform convergence and capacity measures like VC-dimension~\cite{vapnik2015uniform,mohri2018foundations}, often fails to explain this phenomenon~\cite{zhang2017understanding,neyshabur2017exploring}.
Consequently, attention has shifted toward the concept of implicit bias~\cite{neyshabur2017exploring,vardi2023implicit}, which posits that gradient-based optimization methods  favor solutions that generalize well.

 The implicit bias in linear models is well-understood---gradient descent converges to the minimum Euclidean norm solution in regression~\cite{gunasekar2018characterizing,zhang2017understanding} and max-margin solution in classification~\cite{soudry2018implicit}. As for nonlinear model, diagonal linear network, where each layer applies a diagonal linear transformation to the input, is a simple testbed for studying implicit bias.   In~\citet{woodworth2020kernel}, the authors showed that for a two-layer diagonal linear network, the implicit bias of gradient descent is the \( \ell_1 \) norm under infinitesimal initialization, and to the \( L_2 \) norm under large (infinite) initialization. They further demonstrated that similar results hold for deep diagonal linear networks under infinitesimal initialization of a specific direction. 

Beyond characterizing convergence points, there is growing interest in understanding the complete trajectory of the dynamics. It is observed that under small initialization, saddle-to-saddle dynamics is present in many models, such as the diagonal linear network~\cite{pesme2023saddle,jacot2021saddle,berthier2023incremental},
matrix factorization~\cite{li2020towards,bai2024connectivity}, ReLU network~\cite{boursier2022gradient,bantzis2025saddle}. The most relevant work is~\citet{pesme2023saddle}, which, building on the mirror flow technique from~\citet{azulay2021implicit}, proves that under infinitesimal initialization, the gradient flow trajectory of a two-layer diagonal linear network successively transitions from one saddle point to another, ultimately converging to the minimum \(\ell_1\)-norm solution. These saddle-to-saddle transitions are further characterized through a recursive algorithm.

The analysis in~\citet{pesme2023saddle} relies on the explicit form of mirror flow. In~\citet{li2022implicit}, the author establishes the existence of mirror flow in a broader setting known as commuting parametrization, which encompasses diagonal linear networks as a special case. This suggests that the results of~\citet{pesme2023saddle} could potentially be extended to more general architectures, although such a generalization remains an open question.

In this paper, we achieve this generalization via
 a novel technique called the Structural Invariant Manifold~\cite{zhao2026architecture}, which investigates data-independent properties of parametric models.  We extend the analysis of saddle-to-saddle dynamics in~\citet{pesme2023saddle} to both deep diagonal linear networks and  general two-layer diagonal linear networks (defined in Definition~\ref{def:general 2-layer}). As a  direct corollary, we characterize the implicit bias of the two types of diagonal linear network.
 

\subsection{Contributions}
\label{sec:contribution}
\textbf{Training Trajectory:} 
Theorem 1 in~\citet{pesme2023saddle} addresses the saddle-to-saddle dynamics of two-layer diagonal linear networks and deep networks restricted to specific initialization directions. We generalize this in Theorem~\ref{thm: dynamics of diagona linear network} by extending the analysis to general two-layer networks (Definition~\ref{def:general 2-layer}) and establishing results for deep networks under generic initialization directions.

\textbf{Implicit bias:} \citet{gunasekar2017implicit,woodworth2020kernel} established that the implicit bias corresponds to $\ell_1$-norm minimization for standard two-layer diagonal linear networks and deep diagonal linear networks under specific initialization directions. We generalize these findings by characterizing the implicit bias as a modified $\ell_1$-norm minimization for general two-layer diagonal linear networks and  deep diagonal linear networks under generic initialization directions (Corollary~\ref{cor:L-layer}).

\textbf{Mechanism:} We provide a theoretical analysis of the underlying mechanism that leads to the equivalence between the dynamics and the recursive algorithm, identifying the Structural Invariant Manifold~\cite{zhao2026architecture} as the key  for this behavior (Section~\ref{sec:mechanism}).


\section{Related Works}

\textbf{Diagonal Linear Network:} 
Early work on diagonal linear networks focused on establishing the implicit bias of gradient flow~\cite{gunasekar2017implicit,woodworth2020kernel,azulay2021implicit} and mirror flow~\cite{labarriere2024optimization,azulay2021implicit}. Subsequent research has elucidated the geometric nature of these optimization paths, specifically saddle-to-saddle dynamics~\cite{pesme2023saddle,berthier2023incremental,jacot2021saddle,berthier2025diagonal}. More recently, literature has addressed the algorithmic impact of discrete hyperparameters, such as step size~\cite{even2023s,nacson2022implicit}, stochastic noise~\cite{pesme2021implicit}, and momentum~\cite{papazov2024leveraging}.

\textbf{Matrix Factorization:}
Since diagonal linear networks can be viewed as a special case of matrix factorization, related literature is particularly relevant. \citet{gunasekar2017implicit} prove that for commuting observations, gradient descent exhibits an implicit bias toward low nuclear norm solutions. \citet{li2020towards} provide both theoretical and empirical evidence that gradient flow with infinitesimally small initialization behaves like a greedy low-rank algorithm. However, \citet{pesme2023saddle} demonstrate that this greedy low-rank behavior does not extend to diagonal linear networks. \citet{bai2024connectivity} further identify the “connectivity” of observations as a key factor influencing the dynamics.  Extending the technique in this paper to the matrix factorization setting is  an interesting direction for future research.



\section{Preliminary}
\textbf{Analytic Parametric Model and Its Training:} In this paper, we use $F(\vtheta)(\vx)$ to denote a parametric model. $\vtheta\in \sR^M$ is the parameter of the model, $\vx\in \sR^d$ is the input of the model. The output of the model is $F(\vtheta)(\vx)\in \sR$. 
If $F$ is real-analytic for $(\vtheta,\vx)\in \sR^{M}\times \sR^d$, we say $F$ is an analytic parametric model.
To train the parametric model, we collect data $S=\{(\vx_i,y_i)\}_{i=1}^n$. We define MSE loss $L(\vtheta)=\sum_{i=1}^n (F(\vtheta)(\vx_i)-y_i)^2$. We consider its gradient flow
\begin{equation}
    \frac{\rd \vtheta}{\rd t}=-\nabla_{\vtheta} L(\vtheta), \vtheta(0)=\vtheta_0.
\label{eq:gradient flow}
\end{equation}
Here, $\vtheta_0$ is the initialization.

\textbf{Diagonal Linear Network:}
An $L$-layer diagonal linear model  is  $F(\vtheta)(\vx)=\sum_{i=1}^n a_{i,1}a_{i,2}\cdots a_{i,L} x_i$, where $\vx=(x_1,\ldots,x_n)\in \sR^n$, $\vtheta=(a_{i,1},a_{i,2}\ldots ,a_{i,L})_{i=1}^n\in \sR^{nL}$. The training data of diagonal linear model is denoted as $\vX,\vy$, where $\vX$ is an $m\times n$ matrix, $\vy$ is an $m\times 1$ vector. Each row of $\vX$ and $\vy$ is a sample. The MSE loss is 
$L(\vtheta)=\|\vX\vk(\vtheta)-\vy\|_2^2$. Here, $\vk(\vtheta)=(k_i(\vtheta_i))_{i=1}^n,$ where $k_i(\vtheta_i)=a_{i,1}\times a_{i,2}\times \cdots \times  a_{i,L}.$  If $L\geq 3$, we call it the deep diagonal linear network.

\textbf{Notation:} We use $[n]$ to denote the set $\{1,\ldots,n\}$. We use bold letters to denote vectors, and non-bold letters to denote scalars. 
For a vector \(\vv\), we use \(v_i\) denotes the i-th coordinate of  $\vv$. 
For a matrix \(\vX\), we use \(\vX_{ij}\) to denote the element in the \(i\)-th row and \(j\)-th column, 
\(\vX_i\) to denote the \(i\)-th row of \(\vX\), and \(\vX_{:,j}\) to denote the \(j\)-th column of \(\vX\).


\textbf{Infinitesimal initialization:} 
Let $\vtheta_0$ denote a fixed base parameter vector. We set the initialization of the model weights as: $ \vtheta(0) = s \vtheta_0 $ where $s > 0$ is a scalar multiplier. The ``infinitesimal initialization" refers to the asymptotic limit as $s \to 0$.

\textbf{$\ell_1$ Implicit Bias:} Two-layer diagonal linear model is known to exhibit implicit regularization of $\ell_1$ norm under infinitesimal initialization~\cite{gunasekar2017implicit}. Mathematically, if we denote the solution of Equation~\eqref{eq:gradient flow} as $\phi(\vtheta_0,t)$, then we have 
$$\lim_{s \to 0}\lim_{t\to +\infty} \phi(s\vtheta_0,t)=\vk^*,$$
where $\vk^*$ is a solution of the following $\ell_1$ minimization problem:
\begin{equation}
\min _{\vk\in \sR^n} \|\vk\|_1 \quad s.t. \quad \vX\vk=\vy.    
\end{equation}

\section{Main Results}


\begin{definition}[General Two-Layer Diagonal Linear Network]
Consider the model defined by
\[
F(\vtheta)(\vx) = \sum_{i=1}^n k_i(\vtheta_i) x_i,
\]
where $\vtheta_i$ is a vector for each $i\in [n]$,
\( \vtheta = (\vtheta_i)_{i=1}^n \) and \( \vx = (x_i)_{i=1}^n \). Define the vector \( \vk(\vtheta) = (k_i(\vtheta_i))_{i=1}^n \). Given a data matrix \( \vX \in \mathbb{R}^{m \times n} \) and a target vector \( \vy \in \mathbb{R}^m \), we define the loss function as
\[
L(\vtheta) = \|\vX \vk(\vtheta) - \vy\|_2^2.
\]
The parameter \( \vtheta \) is optimized using gradient flow.

For each \( i \in [n] \), we impose the following assumptions on the function \( k_i(\vtheta_i) \):
\begin{itemize}
    \item The function \( k_i(\vtheta_i) \) is real analytic, satisfies \( k_i(\vzero) = 0 \), and has a unique critical point at \( \vtheta_i = \vzero \).
    
    \item Let \( \lambda_1 \leq \lambda_2 \leq \ldots \leq \lambda_d \) denote the eigenvalues of the Hessian \( \nabla^2 k_i(\vzero) \). We assume that
    \[
    \lambda_1 < \min\{0, \lambda_2\}, \quad \lambda_d > \max\{0, \lambda_{d-1}\}.
    \]
    
    \item The function \( k_i(\vtheta_i) \) is unbounded along each of its limit trajectories\footnote{See Lemma~\ref{lemma:limit solution} and Definition~\ref{def:limit trajectory} for detailed explanations. This assumption primarily fails if $k_i(\boldsymbol{\theta}_i)$ is globally bounded from above (i.e., $k_i(\boldsymbol{\theta}_i) < M$ for all $\boldsymbol{\theta}_i$) or below.}.
\end{itemize}
\label{def:general 2-layer}
\end{definition}

\textbf{Summary of Main Results:} Briefly speaking, in Corollary~\ref{cor:L-layer}, we prove the implicit bias  of  deep diagonal linear network  and general two-layer diagonal linear network (defined in Definition~\ref{def:general 2-layer}) under infinitesimal initialization.  The implicit bias is a modified $\ell_1$ norm:
$$R(\vk)=\sum_{i=1}^n (t_i^+k_i^+ +t_i^-k_i^-).$$

The proof of Corollary~\ref{cor:L-layer}  consists of two steps.
In Theorem~\ref{thm: dynamics of diagona linear network}, 
we prove that the dynamics  of the two models is equivalent to Algorithm~\ref{alg:my_algorithm}. In Theorem~\ref{thm: algorithm}, we prove that Algorithm~\ref{alg:my_algorithm} will converge to the solution of 
$$\min _{\vk} R(\vk) \quad s.t. \quad \vX\vk=\vy.$$
 Then the implicit bias is a direct corollary.


\subsection{The Algorithm}

\begin{algorithm}[h]
\caption{Algorithm($\vX,\vy,\{t_i^+\}_{i=1}^n,\{t_i^-\}_{i=1}^n$)}
\label{alg:my_algorithm}
\begin{algorithmic}[1]
\Require Data matrix $\vX \in \mathbb{R}^{m \times n}$,  $\vy \in \mathbb{R}^{m \times 1}$.
\Require  Positive values $\{t_i^+\}_{i=1}^n$ and $\{t_i^-\}_{i=1}^n$.


\State \textbf{Initialization:}
$p \gets 0$, $\vk^{(0)} \gets \vzero$, $\vs^{(0)} \gets \vzero$.

\While{$\vX \vk^{(p)} \neq \vy$}
    \State  $\vu \gets \vX^\T(\vy-\vX \vk^{(p)})$
    \State $I_1 \gets \{i \in \{1, \dots, n\} \mid k_i^{(p)} > 0\}$ 
    \State $I_2 \gets \{i \in \{1, \dots, n\} \mid k_i^{(p)} = 0\}$ 
    \State $I_3 \gets \{i \in \{1, \dots, n\} \mid k_i^{(p)} < 0\}$ 
    
 \For{$i \in I_2$}
    \If{$u_i > 0$, }  $\delta_i \gets ({t_i^+ - s_i^{(p)}})/{u_i}$
    \ElsIf{$u_i < 0$, }  $\delta_i \gets ({t_i^- + s_i^{(p)}})/{|u_i|}$
    \Else \quad
     $\delta_i \gets +\infty$
\EndIf
\EndFor
    
    \State $j \gets \mathrm{argmin}_{i \in I_2} \delta_i $ 
    \State $\vs^{(p+1)} \gets \vs^{(p)} + \delta_j \vu$
    
    \State Find $\vk^{(p+1)}$ as the solution to:
    \begin{align*}
        \min_{\vk} \quad  \|\vX \vk - \vy\|_2^2 \\
        \text{s.t.} \quad & k_i \geq 0, \quad \forall i \in I_1 \\
                     & k_i \leq 0, \quad \forall i \in I_3 \\
                     & k_i = 0, \quad \forall i \in I_2 \setminus \{j\}
    \end{align*}
    
    \State $p \gets p + 1$
\EndWhile

\State \Return $\vk^{(p)}$
\end{algorithmic}
\end{algorithm}

\begin{theorem}[Convergence and Well-Posedness of Algorithm~\ref{alg:my_algorithm}]
Assume the following conditions hold for Algorithm~\ref{alg:my_algorithm}:
\begin{itemize}
    \item The linear system $\vX \vk = \vy$ admits at least one solution.
    \item At each iteration, the index $j = \arg\min_{i \in I_2} \delta_i$ is uniquely defined.
    \item At each iteration \( p \), if there exists an index \( i \in [n] \) such that:  
    (i) \( k_i^{(p)} = 0 \), and  
    (ii) \( s_i^{(p)} = t_i^+ \) or \( s_i^{(p)} = -t_i^- \),  
    then \( u_i \) is assumed to be nonzero.
\end{itemize}
Under these assumptions, Algorithm~\ref{alg:my_algorithm} terminates in a finite number of iterations. Furthermore, the output \( \vk \) of the algorithm is a solution to the following optimization problem:
\begin{equation}
    \min \sum_{i=1}^n \left( t_i^+ k_i^+ + t_i^- k_i^- \right) \quad \text{subject to} \quad \vX \vk = \vy,
    \label{eq:min_L1_norm}
\end{equation}
where \( k_i^+ \) and \( k_i^- \) denote the positive and negative parts of \( k_i \), respectively.
\label{thm: algorithm}
\end{theorem}

\begin{proof}
    The proof is provided in Appendix~\ref{sec: proof of algorithm}.
\end{proof}

\begin{remark}
Algorithm~\ref{alg:my_algorithm} and
Theorem~\ref{thm: algorithm} are  straightforward generalizations of the corresponding result in~\citet{pesme2023saddle}, where the authors assume symmetric growth times, i.e., \( t_i^- = t_i^+ = 1 \) for all \( i \in [n] \).
\end{remark}

\subsection{Dynamics of  diagonal linear network}

\begin{assumption}[Assumptions and Notation for Theorem~\ref{thm: dynamics of diagona linear network}]
\textbf{Shared Notation:}  
We define the loss function as
\[
L(\vtheta) = \|\vX \vk(\vtheta) - \vy\|_2^2.
\]
Let \( \vtheta(t) \) evolve according to the gradient flow dynamics:
\[
\frac{\rd \vtheta}{\rd t} = -\nabla L(\vtheta), \quad \vtheta(0) = \vtheta_0,
\]
and denote the solution by \( \phi(\vtheta_0, t) \).

We assume throughout that the conditions of Theorem~\ref{thm: algorithm} are satisfied. In particular, by Theorem~\ref{thm: algorithm}, Algorithm~\ref{alg:my_algorithm}, when applied to input \( (\vX, \vy, \{t_i^+\}_{i=1}^n, \{t_i^-\}_{i=1}^n) \), terminates after a finite number of iterations. Let \( \{ \vk^{(p)} \}_{p=0}^{p_{\mathrm{max}}} \) denote the sequence of vectors generated at each iteration.

\textbf{Assumptions and Notation for Deep Diagonal Linear Network:}  
Consider an \( L \)-layer diagonal linear network with data matrix \( \vX \in \mathbb{R}^{m \times n} \) and target vector \( \vy \in \mathbb{R}^m \).  Assume $L\geq 3$.
Fix the initialization direction
\[
\vtheta^* = (a_{i,1}^*, \ldots, a_{i,L}^*)_{i=1}^n \in \mathbb{R}^{Ln}.
\]
For each \( i \in [n] \), assume that the minimizer
\[
\mathrm{argmin}_{k \in [L]} (a_{i,k}^*)^2
\]
is unique. Without loss of generality, we assume that this minimum is attained at \( k = L \).

Define the following quantities:
\begin{itemize}
    \item For \( i \in [n] \), \( j \in [L-1] \), define
    \[
    \mu_{i,j}^2 = (a_{i,j}^*)^2 - (a_{i,L}^*)^2.
    \]
    
    \item For each \( i \in [n] \), define the function
    \begin{equation*}
F_i(z) = \int_{a_{i,L}^*}^{z} \frac{1}{\prod_{j=1}^{L-1} \sqrt{x^2 + \mu_{i,j}^2}} \, \rd x.
\end{equation*}

    \item Define
    \[
    t_i^+ = F_i(+\infty), \quad t_i^- = -F_i(-\infty).
    \]
\end{itemize}

\vspace{0.5em}
\textbf{Assumptions and Notation for the General Two-Layer Diagonal Linear Network:}  
For each \( i \in [n] \), let \( \lambda_i^+ \) and \( \lambda_i^- \) denote the largest and smallest eigenvalues, respectively, of the Hessian \( \nabla^2 k_i(\vzero) \). Define
\[
t_i^+ = \frac{1}{\lambda_i^+}, \quad t_i^- = -\frac{1}{\lambda_i^-}.
\]

Fix an initialization direction \( \vtheta^* = (\vtheta_i^*)_{i=1}^n \). Assume that for each \( i \in [n] \), the vector \( \vtheta_i^* \) is not orthogonal to the eigenvectors corresponding to \( \lambda_i^+ \) and \( \lambda_i^- \).

\label{assum:L-layer}
\end{assumption}


\begin{theorem}[Dynamics of Deep Diagonal Linear Networks]
Consider either a deep diagonal linear network or a general two-layer diagonal linear network.  
Suppose Assumption~\ref{assum:L-layer} holds, and let \( \Gamma^s \) denote the trajectory of gradient flow initialized at \( s\vtheta^* \). Then the following statements hold:
\begin{itemize}
    \item For each \( p = 0, 1, \ldots, p_{\mathrm{max}} \), there exists a point \( \vtheta^s \in \Gamma^s \) such that
    \[
    \lim_{s \to 0} \vk(\vtheta^s) = \vk^{(p)}.
    \]
    
    \item The iterated limit
    \[
    \vtheta' = \lim_{s \to 0} \lim_{t \to +\infty} \phi(s\vtheta^*, t)
    \]
    exists, and satisfies
    \[
    \vk(\vtheta') = \vk^{(p_{\mathrm{max}})}.
    \]
\end{itemize}
\label{thm: dynamics of diagona linear network}
\end{theorem}

\begin{proof}

    The proof is provided in Appendix~\ref{sec:proof of dynamics}.
\end{proof}

\begin{corollary}[Implicit Bias of Diagonal Linear Networks]
Under Assumption~\ref{assum:L-layer}, the implicit bias of a deep diagonal linear network or a general two-layer diagonal linear network with infinitesimal initialization is given by
\[
R(\vk) = \sum_{i=1}^n \left( t_i^+ k_i^+ + t_i^- k_i^- \right).
\]
\label{cor:L-layer}
\end{corollary}

\begin{proof}
This is a direct corollary of Theorem~\ref{thm: algorithm} and Theorem~\ref{thm: dynamics of diagona linear network}.
\end{proof}

\begin{remark}
We provide a detailed comparison of Theorem~\ref{thm: dynamics of diagona linear network} and Corollary~\ref{cor:L-layer} with the existing literature in Section~\ref{sec:contribution}. 
\end{remark}




\section{Intuition of the Dynamics}
\label{sec:intuition}
In this section, we use a three-layer diagonal linear network as an illustrative example. The proof of Theorem~\ref{thm: dynamics of diagona linear network} serves as a rigorous formalization of the intuition developed in this section.

Specifically, we analyze the gradient flow dynamics of the model $F(\vtheta)(\vx)=\sum_{i=1}^n a_ib_ic_ix_i$ under a squared error loss $L(\vtheta)$. For a single neuron (dropping the index $i$), the dynamics for its parameters $(a,b,c)$ are given by:
\begin{align*}
    \frac{\rd a}{\rd t} &= bc \cdot v, &
    \frac{\rd b}{\rd t} &= ac \cdot v, &
    \frac{\rd c}{\rd t} &= ab \cdot v,
\end{align*}
where $v = \sum_{k=1}^m \vX_{k}(y_k - F(\vtheta)(\vX_k))$ is a term related to the correlation between the neuron's feature and the prediction error. 


\subsection{The Feature Selection Phase and Distance-Speed Argument}

\label{sec:feature selection}

In the early phase of training, with infinitesimally small initial parameters $\vtheta(0)$, the model output $F(\vtheta)(\vx)$ is negligible. Consequently, $v(t)$ remains approximately constant, $v(t) \approx v_0 = \sum_{k=1}^m \vX_{k}y_k$. 

This constant, $v_0$, can be interpreted as the \textbf{speed} of the neuron's evolution.
By rescaling time as $\tau = v_0 t$, we can isolate the geometry of the parameter trajectory:
\begin{align*}
    \frac{\rd a}{\rd \tau} &= bc, &
    \frac{\rd b}{\rd \tau} &= ac, &
    \frac{\rd c}{\rd \tau} &= ab.
\end{align*}
These dynamics possess two conserved quantities:
\begin{align*}
    a^2(\tau) - c^2(\tau) &= a^2(0) - c^2(0) := \mu_1^2 \\
    b^2(\tau) - c^2(\tau) &= b^2(0) - c^2(0) := \mu_2^2
\end{align*}
Without loss of generality, let us assume $a(0)^2 > b(0)^2 > c(0)^2 > 0$ and $a(0)b(0) > 0$. Substituting the conserved quantities into the dynamic for $c$ yields:
\[
\frac{\rd c}{\rd \tau} = \sqrt{c^2 + \mu_1^2} \sqrt{c^2 + \mu_2^2}.
\]
Now, consider an initialization of scale $s \ll 1$: $a(0)=a_0 s, b(0)=b_0 s, c(0)=c_0 s$. The conserved quantities become $\mu_1^2 = (a_0^2-c_0^2)s^2 := \tilde{\mu}_1^2 s^2$ and $\mu_2^2 = (b_0^2-c_0^2)s^2 := \tilde{\mu}_2^2 s^2$. By defining a scaled parameter $\tilde{c}(\tau) = c(\tau)/s$, its dynamic is:
\[
\frac{\rd \tilde{c}}{\rd \tau} = s \sqrt{\tilde{c}^2 + \tilde{\mu}_1^2} \sqrt{\tilde{c}^2 + \tilde{\mu}_2^2}, \quad \text{with} \quad \tilde{c}(0) = c_0.
\]
A neuron becomes ``significant" when its parameters grow from $\mathcal{O}(s)$ to $\mathcal{O}(1)$. In the scaled frame, this corresponds to $\tilde{c}$ growing from $c_0$ to $\pm\infty$. The time required for this growth (the blow-up time) can be found by separating variables. For instance, the time to grow to positive significance is:
\[
\begin{aligned}
\tau^+ &= \int_{c_0}^{\infty} \frac{\rd \tilde{c}}{s \sqrt{\tilde{c}^2 + \tilde{\mu}_1^2} \sqrt{\tilde{c}^2 + \tilde{\mu}_2^2}}\\& = \frac{1}{s} \int_{c_0}^{\infty} \frac{\rd x}{\sqrt{x^2 + \tilde{\mu}_1^2} \sqrt{x^2 + \tilde{\mu}_2^2}}.
\end{aligned}
\]
This structure motivates the core intuition. We define the \textbf{distance} a neuron must travel as the integral, which depends only on the initial \emph{direction} $(a_0, b_0, c_0)$:
\[
D^+ = \int_{c_0}^{\infty} \frac{\rd x}{\sqrt{x^2 + \tilde{\mu}_1^2} \sqrt{x^2 + \tilde{\mu}_2^2}}.
\]
The \textbf{time} to grow to significance in the original time $t$ is obtained by un-scaling $\tau^+$:
\begin{equation}
t^+ = \frac{\tau^+}{|v_0|} = \frac{1}{|v_0|s} D^+.
\label{eq:distance speed}
\end{equation}

\textbf{Distance-Speed Argument:}
Equation~\eqref{eq:distance speed} reveals that the time for a neuron to become significant is proportional to a ``distance'' $D^+$ (determined by its initialization) and inversely proportional to its ``speed'' $|v_0|$ (determined by its feature's correlation with the target). 
The neurons with \textbf{shortest time} 
will be $\fO(1)$ first, and therefore it will be
chosen as the grown feature.

\subsection{The Learning Phase and Sign-Locking of Features}

\textbf{Learning Phase:}
Once a feature, say $k_j = a_jb_jc_j$, has grown to a significant magnitude at time $t_{\min}$, the dynamics of the system change. The approximation that the learned vector $\vk$ is near zero no longer holds. The system now enters a ``learning phase" where the gradient descent dynamics actively adjust the significant features to minimize the loss function $L(\vtheta)$. 

\textbf{Mismatch of Time Scales of Two Phases:} There exists a mismatch in the time scales of the feature selection and learning phases. Specifically, the feature selection phase persists for a duration on the order of \(\mathcal{O}\left(\frac{1}{s}\right)\), whereas the duration of the learning phase is independent of \(s\). Consequently, when \(s\) is sufficiently small, the time spent in the learning phase becomes negligible in comparison to that of the feature selection phase. This implies that, during training, no new features are developed away from zero.

\textbf{Sign-Locking:}
The key insight is that once $k_j(t)$ has become significantly non-zero, its sign is effectively locked for a very long time. Consider the case where $k_j(t)$ has grown to be positive, meaning $a_j(t),b_j(t),c_j(t)$ are all large. Since we have conserved quantities
\begin{align*}
    a_j^2(\tau) - c_j^2(\tau) &= a_j^2(0) - c_j^2(0)  \\
    b_j^2(\tau) - c_j^2(\tau) &= b_j^2(0) - c_j^2(0), 
\end{align*}
therefore, the parameter $\vtheta_i(\tau):=(a_j(\tau),b_j(\tau),c_j(\tau))$ is \textbf{restricted to a curve} during training. As a consequence,
for $k_j(t)$ to flip its sign and become negative, $a_j(t),b_j(t),c_j(t)$ would first have to decrease back to its initialization, which is $\fO(s)$ scale. So
the time for $k_j(t)$ to flip its sign and become negative is of order $\fO(\frac{1}{s})$. Since we assume $s$ is a sufficiently small number representing the scale of  initialization, this required time is enormous. 

This reveals a mismatch in time scales during the learning dynamics. Consider the second learning phase. Let \(k_j\) and \(k_l\)  denote the feature chosen  at the first and second feature selection phase, respectively. We compare the following two time scales:

\begin{itemize}
    \item The time required for \(k_j\) to flip its sign is of order \(\frac{1}{s}\), which diverges to infinity as \(s \to 0\).
    \item The time required for \(k_j\) and \(k_l\) to reach a minimum of the loss function \(L(\boldsymbol{\theta})\) without changing their signs is approximately independent of \(s\), since both coordinates have already moved away from zero.
\end{itemize}

Consequently, when \(s\) is sufficiently small, \(k_j\) and \(k_l\) reach the minimum long before \(k_j\) is able to flip its sign. This results in a learning phase that can be effectively modeled as a sign-constrained optimization problem. For instance, if \(k_j\) becomes positive during the first learning phase, then the second learning phase can be described by the following constrained optimization:
\[
\min_{k_j, k_l} \left\| \vX_{:,j} k_j + \vX_{:,l} k_l - \vy \right\|_2^2 \quad \text{subject to} \quad k_j \geq 0.
\]

\subsection{Comprehensive Dynamics}


\textbf{Inter-Phase Continuity of Feature Selection:}  
Although certain neurons are not selected during a given feature selection phase, they still undergo non-trivial movement and settle at specific positions within the interval \((-t^-, t^+)\), as characterized by Equation~\eqref{eq:distance speed}. Given that the time scale of the subsequent feature learning phase is negligible, it does not significantly affect the final positions of these neurons. As a result, in the next feature selection phase, each non-activated neuron resumes its dynamics from the position reached at the end of the previous selection phase.

\textbf{Feature Selection Drives Subsequent Learning:}  
The training process alternates between feature selection and feature learning phases, each exerting influence on the other. During a feature selection phase, a new feature is typically activated. Consequently, the subsequent learning phase jointly optimizes this newly selected feature along with those previously selected.

\textbf{Learning Shapes Future Selection Dynamics:}  
Conversely, each learning phase modifies the correlation vector \(\vv = \vX^\T(\vy-\vX\vk^{(p)} )\), which governs the “speed” of neurons in the ensuing feature selection phase. As a result, the selection behavior in each feature selection phase is inherently influenced by the convergence point of the preceding learning phase.

\textbf{Comprehensive Dynamics:}  
The training process proceeds iteratively through alternating phases. At each iteration \(p\), the following steps are performed:

\begin{enumerate}
    \item \textbf{Feature Selection Phase:}  
    A new neuron \(j_p\) is selected from the inactive set based on the minimal activation time, computed as the ratio of ``distance'' to ``speed''.  
    The distance corresponds to the remaining gap inherited from the end of the previous selection phase, while the speed is determined by the correlation vector \(\vv = \vX^\T(\vy-\vX\vk^{(p)})\).
    
    \item \textbf{Learning Phase:}  
    The current active set of neurons \(\mathcal{A}_p = \{j_1, \dots, j_p\}\) jointly minimizes the loss function, subject to individual sign constraints on their corresponding coefficients.
    
    \item \textbf{Termination:}  
    The procedure repeats until convergence, i.e., the gradient flow converges to its global minimum.
\end{enumerate}

This framework of sequential feature activation and joint optimization defines the overall training trajectory and is exactly Algorithm~\ref{alg:my_algorithm}.

\section{Mechanism Behind the Dynamics}
\label{sec:mechanism}
In Section~\ref{sec:intuition}, we used a three-layer diagonal linear network as an illustrative example to explain the validity of Theorem~\ref{thm: dynamics of diagona linear network}. In this section, we aim to uncover a more general underlying mechanism that helps us understand the dynamics of diagonal linear networks.

\subsection{Structural Invariant Manifold (SIM)}

\begin{definition}[Structural Invariant Manifold (SIM)~\cite{zhao2026architecture}]
Let $F(\vtheta)(\vx),\vtheta\in \sR^M, \vx\in \sR^d$
be an analytic parametric model. For an immersed submanifold $\fM\subset \sR^M$, we say $\fM$ is a \textbf{structural invariant manifold} if it is invariant under $-\nabla_{\vtheta}L(\vtheta)$ in Equation~\eqref{eq:gradient flow} for any real analytic loss function $\ell:\sR\times \sR \to \sR$ and dataset $S$\footnote{See Definition~\ref{def:invariant manifold} for the definition of invariant manifold.}.
    
\end{definition}

\begin{definition}[orbit, page 33 of~\citet{jurdjevic1997geometric}]
    Let $\fF$ be a family of analytic vector fields on an analytic manifold $\fM$. Let $G=G(\fF)$ be the group (pseudogroup) of diffeomorphisms (local diffeomorphisms) generated by $\{e^{tX}\mid t\in\sR, X\in\fF\}$ under composition. For any $\vtheta\in \fM$, we define the \textbf{orbit} of $\fF$ through $\vtheta$ as $\{g(\vtheta)\mid g\in G\}$, which we denote by $O_{\fF}(\vtheta)$. \footnote{There is  a detailed explanation of orbit and its properties in chapter 2 of~\citet{jurdjevic1997geometric}.}
\label{def:orbit}
\end{definition}


\begin{theorem}[SIMs of $F$ are orbit unions of $\fF$~\cite{zhao2026architecture}]
 Let \( F(\vtheta)(\vx),\vtheta\in \sR^M, \vx\in \sR^d \) be an analytic parametric model. Let \( \fF = \left\{ \nabla_{\vtheta} F(\cdot)(\vx) \mid \vx \in \mathbb{R}^d \right\} \). 
Then 
\begin{itemize}
    \item Each SIM of $F$ is union of orbits of $\fF$.
    \item Each orbit of $\fF$ is an SIM.
\end{itemize}
\label{def:sim}
\end{theorem}

\begin{remark}
 Theorem~\ref{def:sim} implies that orbit is the ``smallest unit'' of SIM.
\end{remark}

\begin{theorem}[SIM of diagonal linear network]
Consider the model $F(\vtheta)(\vx)=\sum_{i=1}^n k_i(\vtheta_i)x_i, \vtheta=(\vtheta_i)_{i=1}^n, \vx=(x_i)_{i=1}^n$.  Assume that for all $i\in [n]$, the function \( k_i(\vtheta_i) \) is real analytic, and has a unique critical point at \(  \vzero \). Define 
\begin{itemize}
    \item $\fF = \left\{ \nabla_{\vtheta} F(\cdot)(\vx) \mid \vx \in \mathbb{R}^d \right\}$
    \item $\fF_i = \left\{ \nabla k_i(\cdot)\right\}, i\in [n].$ 
\end{itemize}
Then for each $\vtheta=(\vtheta_i)_{i=1}^n$, we have 
$$O_{\fF}(\vtheta)= \prod_{i=1}^n O_{\fF_i}(\vtheta_i).$$
Besides, for each $i\in [n]$, the following holds:
\begin{itemize}
    \item If $\vtheta_i=\vzero$, then $O_{\fF_i}(\vtheta_i)=\{\vzero\}.$
    \item If $\vtheta_i=\vzero$, then $O_{\fF_i}(\vtheta_i)$ is a simple curve\footnote{ Here, a simple curve refers to an analytic curve that does not intersect itself.}.
\end{itemize}

\label{thm:SIM}
\end{theorem}

\begin{proof}
By calculation, we have
$$\nabla_{\vtheta}F(\vtheta)(\vx)=(\nabla k_i(\vtheta_i)x_i)_{i=1}^n.$$
So we have 
$$\fF=\prod_{i=1}^n \fF_i.$$
Therefore, it holds that 
$$O_{\fF}(\vtheta)=\prod_{i=1}^n O_{\fF_i}(\vtheta_i).$$
If $\vtheta_i=0$, then $\nabla k_i(\vtheta_i)=\vzero.$ So $O_{\fF_i}(\vtheta_i)=\{\vzero\}.$ If $\vtheta_i\neq 0$, by assumption, $\nabla k_i(\vtheta_i)\neq 0$. By Hermann--Nagano Theorem (Theorem 6 in Section 2 of~\citet{jurdjevic1997geometric}), $O_{\fF_i}(\vtheta_i)$ is a real analytic submanifold, and $\dim(O_{\fF_i}(\vtheta_i))=1$.

It is readily to verify that, if $\vtheta_i\neq \vzero$, on $O_{\fF_i}(\vtheta_i)$, the value of $k_i(\cdot)$ is strictly increasing. So   $O_{\fF_i}(\vtheta_i)$ is non-self intersecting.

\end{proof}

\subsection{Implications of SIM in Dynamics}

\textbf{SIM Induces Two-Phase Dynamics:}  
In Theorem~\ref{thm:SIM}, we show that for each \( i \in [n] \), if the initial parameter satisfies \( \vtheta_i = \vzero \), then \( O_{\fF_i}(\vtheta_i) = \{\vzero\} \). This invariance implies that parameters initialized exactly at zero remain stationary under gradient flow indefinitely. In our setting, however, parameters are initialized infinitesimally close to zero rather than exactly at zero. Then such parameters require an infinite amount of time to move significantly away from the origin. Consequently, the learning dynamics exhibit a natural separation into two distinct phases: an initial feature selection phase, during which parameters remain near zero for an extended period, followed by a learning phase that unfolds over a finite time scale. 


\textbf{SIM Induces Sign Constraint During the Learning Phase:}  
By Theorem~\ref{thm:SIM}, under gradient flow dynamics, if the initialization \( \vtheta^* = (\vtheta_i^*)_{i=1}^n \) satisfies \( \vtheta_i^* \neq \vzero \), then for each \( i \in [n] \), the trajectory \( \vtheta_i(t) \) evolves along a simple curve for all \( t \in \mathbb{R} \). Suppose, for contradiction, that \( \vtheta_i(t) \) undergoes a sign change during training; that is, there exist times \( t_1 < t_2 \) such that \( \vtheta_i(t_1) > \delta > 0 \) and \( \vtheta_i(t_2) < \delta' < 0 \). Since \( \vtheta_i(t) \) evolves along a continuous simple curve, the intermediate value theorem implies the existence of some \( \hat{t} \in (t_1, t_2) \) such that \( \vtheta_i(\hat{t}) = \vtheta_i^* \).

Now, if the initialization \( \vtheta_i^* \) is chosen to be infinitesimally small, then reaching this value along the trajectory requires an infinite amount of time. However, the learning phase proceeds over a finite time horizon. As a result, such a sign change cannot occur within the learning phase. This implies that, due to the structure imposed by SIM, the sign of each \( \vtheta_i(t) \) is effectively preserved throughout the learning phase.

\subsection{Mechanism of Incremental Learning}  
In this context, \emph{incremental learning} refers to the phenomenon whereby \textit{only one neuron} is selected during the feature selection phase.

Consider model of the form $F(\vtheta)(\vx)=\sum_{i=1}^n k_i(\vtheta_i)x_i$.
For each \( i \in [n] \), the gradient flow dynamics of the parameter \( \vtheta_i \) are given by:
\begin{equation}
    \frac{\rd \vtheta_i}{\rd t} = \nabla k_i(\vtheta_i) \, u_i(\vtheta),
\end{equation}
where \( \vu(\vtheta) = \vX^\T (\vy-\vX \vk(\vtheta)) \). During the feature selection phase, the change in \( \vtheta \) is small due to the scale of initialization. Consequently, \( u_i(\vtheta) \) remains approximately constant and can be approximated by a fixed value \( u_i^* \). Therefore, the dynamics simplify to:
\begin{equation}
    \frac{\rd \vtheta_i}{\rd t} \approx \nabla k_i(\vtheta_i) \, u_i^*.
\end{equation}

Under this approximation, \( u_i^* \) can be interpreted as the effective “speed” at which \( \vtheta_i \) evolves along the  curve \( O_{\fF_i}(\vtheta_i) \).

As established in Section~\ref{sec:feature selection} and Lemma~\ref{lemma: short time leads to zero}, for deep diagonal linear networks and general two-layer diagonal linear networks with an initialization scale $s$, there exists a scaling function $h(s)$ such that $\lim_{s \to 0} h(s) = +\infty$. This function governs the temporal dynamics required for a neuron to escape the initialization regime. Specifically, the time required for a neuron to attain an $O(1)$ magnitude in the positive or negative direction is approximately $t_i^+ h(s)$ and $t_i^- h(s)$, respectively. For instance:
\begin{itemize}
    \item \( h(s) = \frac{1}{s} \) for three-layer diagonal linear networks,
    \item \( h(s) = -\log s \) for general diagonal linear networks.
\end{itemize}

Define the quantity
\[
\delta_i := 
\begin{cases} 
\frac{t_i^+}{u_i^*}, & \text{if } u_i^* > 0, \\[8pt]
\frac{t_i^-}{|u_i^*|}, & \text{if } u_i^* < 0.
\end{cases}
\]
$\delta_i$ represents the effective time required for neuron \( i \) to $O(1)$, taking into account both its growth direction and speed. The neuron with the smallest \( \delta_i \) will be the one that activates first and is thus selected during the feature selection phase. Under generic conditions on the data \( \vX \) and \( \vy \), the minimum of \( \delta_i \) is attained uniquely, implying that only one neuron is selected in this phase.

\section{Experiments}

We consider  the model  $F(\vtheta)(\vx)=\sum_{i=1}^4 k_i(\vtheta_i)x_i$, where
$$k_i(\vtheta_i)=2a_ib_i+(a_i^2+b_i^2+c_i^2)c_i \, , i=1,2 ;$$ $$k_i(\vtheta_i)=a_i\tanh(a_i)-(e^{b_i}-1)^2\,, i=3,4.$$
One can readily verify that this model satisfies Definition~\ref{def:general 2-layer}. Ideally, to verify Theorem~\ref{thm: dynamics of diagona linear network}, we would require infinitesimal initialization and gradient flow, which are impractical to implement in experimental settings. As a practical alternative, we initialize the model with a sufficiently small scale of \(10^{-60}\), and employ a relatively large learning rate of $0.1$ to accelerate the training of gradient descent.

The data matrix and label vector are given by:
\begin{equation}
\vX = \begin{pmatrix}
1 & 0.5 & 0.7 & 0 \\
0.5 & 1 & 0.1 & 0.7
\end{pmatrix}, \quad 
\vy = \begin{pmatrix}
1 \\
0
\end{pmatrix}.
\label{eq:data matrix}
\end{equation}

As illustrated in Figure~\ref{fig:experiment}, the training dynamics demonstrate transitions between successive saddle points. The three dashed lines in the second panel indicate the values of \(k_i(\cdot)\) following each learning phase. In Appendix~\ref{sec:details of figure}, we compute the values of \(\vk^{(p)}\) in Algorithm~\ref{alg:my_algorithm} for \(p = 1, 2, 3\), and confirm that these computed values closely correspond to those represented by the dashed lines in the figure.

\begin{figure}[h]
    \centering
    \includegraphics[width=0.9\linewidth]{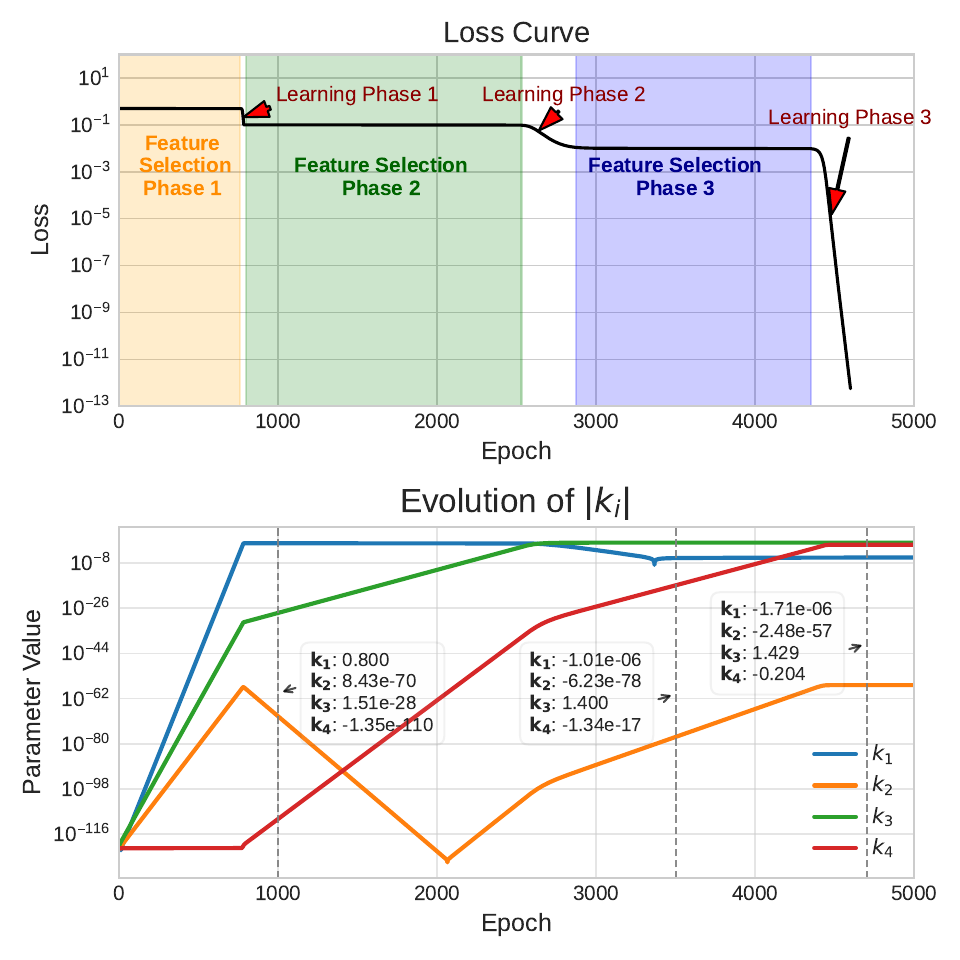}
    \caption{The model is trained via gradient descent. The initialization scale and learning rate are set to \( 10^{-60}\) and $0.1$, respectively. We utilize the \textit{mpmath} library to support high-precision computations. In the second figure, the boxes indicate the values of \(k_i\) at the positions of the vertical dashed lines. }
    \label{fig:experiment}
\end{figure}

\section{Limitations and Discussion}



\subsection{Limitations}

\textbf{Infinitesimal Initialization:}  
Our analysis is confined to the regime of infinitesimal initialization. While this setting facilitates theoretical tractability, it presents two notable drawbacks:  
(i) Infinitesimal initialization is not practical in real-world applications, where parameters are typically initialized with finite variance;  
(ii) Prior works on implicit bias, such as~\citet{woodworth2020kernel}, is able to characterize the implicit bias under all initialization scales.

\textbf{Gradient Flow Assumption:}  
The theoretical analysis in this paper is conducted under the assumption of continuous-time gradient flow dynamics. However, in practice, optimization is performed using discrete-time algorithms such as gradient descent or adaptive methods like Adam. The extent to which our results extend to these more realistic optimization settings remains an open question.

\subsection{Discussion}

\paragraph{Technique:} We highlight that the technique employed in this work relies on the concept of the Structural Invariant Manifold, which generally arises in nonlinear models such as matrix sensing and neural networks~\cite{zhao2026architecture,simsek2021geometry,liu2024symmetry,marcotte2023abide}. The potential outcomes of this study may be extendable to more complex models such as matrix sensing. Investigating this generalization is an interesting direction for future research.

\paragraph{Similarity and difference between $\ell_1$ and modified $\ell_1$ minimization:}
 Both the standard and modified $\ell_1$ norms fundamentally promote sparsity. Their difference is that the
modified $\ell_1$ norm can alter the recovered support (i.e., which specific features the model chooses to select).
To illustrate this, consider the data matrix presented in Equation~\eqref{eq:data matrix}.
\begin{itemize}
    \item Standard $\ell_1$ minimization yields the solution: $k_3=1.43, k_4=-0.20$, with $k_1=k_2=0$ (feature support: $\{3, 4\}$).
    \item Modified $\ell_1$ minimization (applying a $0.5$ penalty weight to $|k_1|$) shifts the optimal solution to: $k_1=1.0, k_4=-0.71$, with $k_2=k_3=0$ (feature support: $\{1, 4\}$).
\end{itemize}
This demonstrates that though the solution remains sparse, the actual features learned by the network may differ.

\paragraph{Gap between gradient flow (GF) and gradient descent (GD):} Whether the SIM remains exactly invariant under GD or SGD depends on its geometric curvature:
\begin{itemize}
    \item Affine (Flat) SIMs: If a SIM is an affine subspace, it remains strictly invariant under GD/SGD. Because the tangent space is globally constant, discrete steps do not cause the trajectory to leave the manifold. For instance, in a standard diagonal linear network parameterized by $k_i(\vtheta_i) = a_i b_i$, the affine SIM defined by $\{ (\vtheta_i)_{i=1}^n \mid a_1 = b_1 = 0\}$ is invariant under GD and SGD.
    \item Curved SIMs: If a SIM is curved, it generally loses its exact invariance under GD/SGD. Taking a finite discrete step along the tangent space of a curved manifold naturally causes the parameter to diverge slightly from the exact manifold. For example, the curved SIM defined by $\{(\vtheta_i)_{i=1}^n \mid a_1^2 - b_1^2 = 1\}$ is not exactly invariant under GD.
\end{itemize}
This theoretical gap explains a specific phenomenon observed in the training dynamics of Figure~\ref{fig:experiment} (optimized via GD, learning rate $\eta=0.1$). Between epochs $3000$ and $4000$, $k_1$ changes sign but fails to return to its exact initial scale ($10^{-120}$). This behavior is directly driven by the loss of exact invariance on the curved SIM caused by the discretization. 


\section{Conclusion}

In this paper, we investigated the training dynamics of deep diagonal linear networks and general two-layer diagonal linear networks. We showed that their training dynamics can be equivalently described by Algorithm~\ref{alg:my_algorithm}. Furthermore, we proved that Algorithm~\ref{alg:my_algorithm} converges to the solution of a modified \( \ell_1 \) norm minimization problem. As a result, we established that the implicit bias of both types of diagonal linear networks under infinitesimal initialization corresponds to a modified \( \ell_1 \) norm.

In addition, we analyzed the underlying mechanisms driving these dynamics and identified the Structural Invariant Manifold (SIM) as the key geometric structure governing the learning process.

\section*{Acknowledgements}
This work was supported by the
National Key R\&D Program of China (Grant No.~2022YFA1008200), the National Natural Science Foundation of China (Grant No.12571567), the Natural Science Foundation of Shanghai (Grant No. 25ZR1402280).

\section*{Impact Statement}
This paper presents work whose goal is to advance the field of Machine Learning. There are many potential societal consequences of our work, none which we feel must be specifically highlighted here.

\bibliography{camera_ready}
\bibliographystyle{icml2026}

\newpage
\appendix
\onecolumn

\section{Proof of Theorems}

\subsection{Proof of Theorem~\ref{thm: algorithm}}
\label{sec: proof of algorithm}
We first prove a lemma.
\begin{lemma}
Consider Algorithm~\ref{alg:my_algorithm} under assumptions of Theorem~\ref{thm: algorithm}. 
Let $J^{(p)}$ be the union of index set $I_1\cup I_3\cup \{j\}$ at step $p$. Then $\{\vX_{:,j}\}_{j\in J^{(p)}}$ is linearly independent for each $p.$  \label{lemma: linearly indpendent}
\end{lemma}

\begin{proof}
We prove by induction. For $p=0$, $J^{(p)}=\{j\}$, where  $j=\mathrm{argmin}_{i\in [n]} \delta_i$. Then we have $u_j\neq 0$. Since $$u_j=(\vX^\T\vy)_j=\langle \vX_{:,j},\vy\rangle,$$
we have $\vX_{:,j}\neq \vzero.$ Therefore the statement holds for $p=0.$ Assume the statement holds for some natural number $p$. We want to prove it for $p+1$.

Let $j$ be the neuron chosen at step $p$. Let $I_1^{(p)} = \{i \in [n] \mid k_i^{(p)} > 0\}$, $I_2^{(p)} = \{i \in [n] \mid k_i^{(p)} = 0\}$, $I_3^{(p)} = \{i \in [n] \mid k_i^{(p)} < 0\}.$ 
By the definition of Algorithm~\ref{alg:my_algorithm}, $\vk^{(p+1)}$ is the solution of the optimization problem
\begin{equation}
\min_{\vk} \quad  \|\vX \vk - \vy\|_2^2 \quad 
        \text{s.t.} \quad  k_i \geq 0, \forall i \in I_1^{(p)}, \quad 
                      k_i \leq 0,  \forall i \in I_3^{(p)},\quad k_i = 0,  \forall i \in I_2^{(p)} \setminus \{j\}
                    \label{eq:linear independent}
    \end{equation}

Define $I_1^{(p+1)} = \{i \in [n] \mid k_i^{(p+1)} > 0\}$, $I_2^{(p+1)} = \{i \in [n] \mid k_i^{(p+1)} = 0\}$, $I_3^{(p+1)} = \{i \in [n] \mid k_i^{(p+1)} < 0\}.$
Define $\vu=\vX^\T(\vy-\vX\vk^{(p+1)})$.

Since $\vk^{(p+1)}$ solves optimization problem~\ref{eq:linear independent}, therefore for each $i\in I_1^{(p+1)}\cup I_3^{(p+1)}$, we have $u_i=0.$ Let $j'$ be the neuron chosen at step $p+1$. 

If $\vX_{:,j'}$ is in linear span of $\{\vX_{:,i}\}_{i\in I_1^{(p+1)}\cup I_3^{(p+1)} }$, then  it is readily verifiable that $u_{j'}=0$. So $\delta_{j'}=+\infty$. This contradicts the fact that $j'$-th neuron is chosen at step $p+1$. So $\vX_{:,j'}$ is not in linear span of $\{\vX_{:,i}\}_{i\in I_1^{(p+1)}\cup I_3^{(p+1)} }.$

By our assumption, $\{\vX_{:,j}\}_{j\in J^{(p)}}$ is linearly independent. Since we have  $I_1^{(p+1)}\cup I_3^{(p+1)} \subset J^{(p)}$,  $\{\vX_{:,i}\}_{i\in I_1^{(p+1)}\cup I_3^{(p+1)} }$ is linearly independent. 

Since $\vX_{:,j'}$ is not in the linear span of $\{\vX_{:,i}\}_{i\in I_1^{(p+1)}\cup I_3^{(p+1)} },$ $\{\vX_{:,j}\}_{j\in J^{(p+1)}}$ is linearly independent. By mathematical induction, the lemma is proved.
\end{proof}

In the following we prove Theorem~\ref{thm: algorithm}.
\begin{proof}

\textbf{Well-Posedness:} At each step $p$, by assumption, the index $j = \arg\min_{i \in I_2} \delta_i$ is uniquely defined. So $j$ is well-posed.

Let $J=I_1^{(p)} \cup I_3^{(p)}$.
By Lemma~\ref{lemma: linearly indpendent}, $\{X_{:,j}
\}_{j\in J}$ is linearly independent. Recall that $\vk^{(p)}$ is the solution of
\begin{equation}
\min_{\vk} \quad  \|\vX \vk - \vy\|_2^2 \quad 
\text{s.t.} \quad  k_i \geq 0, \forall i \in I_1^{(p)}, \quad 
k_i \leq 0,  \forall i \in I_3^{(p)},\quad k_i = 0,  \forall i \in I_2^{(p)} \setminus \{j\}.
\label{eq:algorithm theorem, origial}
\end{equation}
 Then $\vk_J^{(p)}$ is the solution of
\begin{equation}
\min_{\vk_J} \quad  \|\vX_{:,J} \vk_J - \vy\|_2^2 \quad 
\text{s.t.} \quad  k_i \geq 0, \forall i \in I_1^{(p)}, \quad 
k_i \leq 0,  \forall i \in I_3^{(p)}.
\label{eq:algorithm theorem}
\end{equation}
Since $\{X_{:,j}
\}_{j\in J}$ is linearly independent, optimization problem~\eqref{eq:algorithm theorem} is strictly convex. So the solution of~\eqref{eq:algorithm theorem} is unique. Therefore, the solution of~\eqref{eq:algorithm theorem, origial} is unique. So $\vk^{(p+1)}$ is well defined.

\textbf{Terminates in finite iterations:}

    Note that the error $\epsilon^{(p)}=\|\vX\vk^{(p)}-\vy\|_2^2$ strictly decreases over $p$. Besides, there only exists finite possible $\epsilon$ of 
the minimization problem $$\min \|\vX \vk-\vy\|_2^2, \quad s.t. \quad \vk_i\geq 0,i\in I_1;\vk_i \leq 0,i\in I_3; \vk_i=0,i\in I_2\setminus{j} $$
for different $I_1,I_3,I_2,j$. So $\epsilon^{(p)}$ must reach  its minimum after finite iterations. Since $\vX\vk=\vy$ has solution, so $\epsilon^{(p)}$ reach $0$ after finite iterations.

\textbf{Convergence to Optimization Problem~\eqref{eq:min_L1_norm,prove}:}
Consider the optimization probelm
\begin{equation}
    \min \sum_{i=1}^n \left( t_i^+ k_i^+ + t_i^- k_i^- \right) \quad \text{subject to} \quad \vX \vk = \vy.
    \label{eq:min_L1_norm,prove}
\end{equation}
For $\vk=(k_1,\ldots,k_n)\in \sR^n$, define $\partial(\vk)=(\partial_i(k_i))_{i=1}^n$, where 
\[
\partial_i(k_i) = 
\begin{cases}
t_i^+, & \text{if } k_i > 0 \\
-t_i^-, & \text{if } k_i < 0 \\
[-t_i^-,\ t_i^+], & \text{if } k_i = 0
\end{cases}
\]

By calculation, the KKT condition of ~\eqref{eq:min_L1_norm,prove} is 
$$\vzero\in \partial (\vk)+\vX^\T \vlambda, \quad \vX\vk=\vy.$$

Since the optimization problem described in Equation~\eqref{eq:min_L1_norm,prove} is convex, and it has only equal constraint, then the solution of KKT condition is equivalent to solution of original problem.

For each $p$, let $\vk^{(p)}$ and $\vs^{(p)}$ be the value of $\vk$ and $\vs$ at step $p$. 
Let $\vk^*$ be the output of the algorithm. We now verify that $\vk^*$ satisfies the KKT condition.

From the algorithm, one sees that for each $p$, we have the following: \begin{itemize}
    \item If $s_i^{(p)}=t_i^+$, then $k_i^{(p)}\geq 0$.
    \item If $s_i^{(p)}=-t_i^-$, then $k_i^{(p)}\leq 0$.
    \item If $s_i^{(p)}\in (-t_i^-,t_i^+)$, then $k_i^{(p)}=0$.
\end{itemize}

So $\vs^{(p)}\in \partial (\vk^{(p)})$. Besides, by the update of $\vs^{(p)}$, one sees that $\vs^{(p)}\in \text{Img}(\vX^\T)$. So there exists $\vlambda^{(p)} \in \sR^{m}$ such that $\vs^{(p)}=\vX^\T(-\vlambda^{(p)})$. So 
$$\vs^{(p)}+\vX^\T \vlambda^{(p)}=\vzero,\quad  \vs^{(p)}\in \partial (\vk^{(p)}).$$

So the condition 
$$\vzero\in \partial (\vk)+\vX^\T \vlambda$$
holds in each step. Moreover, we have $\vX\vk^*=\vy$.
So $\vk^*$ satisfies the KKT condition. So $\vk^*$ is a solution of~\eqref{eq:min_L1_norm,prove}.
\end{proof}




\subsection{Proof of Theorem~\ref{thm: dynamics of diagona linear network}}
\label{sec:proof of dynamics}

\begin{lemma}[\citet{li2020towards}]
Assume $g(\vtheta):\sR^M\to \sR$ is real analytic, and $g(\vzero)=0,\nabla g(\vzero)=\vzero$. Let $\vH=\nabla^2 g(\vzero)$. Assume $\vH$ has a unique largest eigenvalue $\lambda_{1}>0$ with corresponding eigenvectors $\vv_{1}$. 
Denote the solution of
$$\frac{\rd \vtheta}{\rd t}=\nabla g(\vtheta), \vtheta(0)=\vtheta_0$$
as $\phi(\vtheta_0,t)$. Then for  $\vtheta_1\in \sR^d$ such that $\langle \vtheta_1,\vv_1\rangle > 0$, the limit 
$h(\vtheta_1,t):=\lim_{s\to 0} \phi(s\vtheta_1,t+\frac{1}{\lambda_{1}}\log \frac{1}{s})$ exists, and $h(\vtheta_1,t)\neq \vzero.$ Moreover, the trajectory $\Gamma(\vtheta_1):=\{h(\vtheta_1,t)\mid t\in \sR\}$ is independent of $\vtheta_1$ as long as $\langle \vtheta_1,\vv_1\rangle > 0$. Besides, the same statements also hold for $\langle \vtheta_1,\vv_1\rangle < 0$.
\label{lemma:limit solution}
\end{lemma}
\begin{proof}
The lemma is proved in~\citet{li2020towards} in their Theorem 5.3. 
\end{proof}

\begin{definition}[limit trajectory]
In general two-layer diagonal linear network, we apply Lemma~\ref{lemma:limit solution} by setting $g(\vtheta)= k_i(\vtheta_i)$.
We use   $\Gamma_i^{++},\Gamma_i^{+-}$ to denote the $\Gamma(\vv_1)$ and $\Gamma(-\vv_1)$ in Lemma~\ref{lemma:limit solution}, respectively. One can  also consider the limit trajectory of $$\frac{\rd \vtheta}{\rd t}=-\nabla k_i(\vtheta_i), \vtheta(0)=\vtheta_0$$
when we assume  $\vH$ has a unique smallest eigenvalue $\lambda_{2}<0$ and corresponding eigenvector $\vv_2$. We use    $\Gamma_i^{-+},\Gamma_i^{--}$ to denote the $\Gamma(\vv_2)$ and $\Gamma(-\vv_2)$. The four trajectories, $\Gamma_i^{++},\Gamma_i^{+-},\Gamma_i^{-+},\Gamma_i^{--}$ are called limit trajectories of neuron $i$.

\label{def:limit trajectory}
\end{definition}

\noindent\textbf{Example of limit trajectory:} Consider $k_i(\vtheta)=a^2-b^2, \vtheta=(a,b)\in \sR^2$. Then $\Gamma_i^{++}=\{(a,b)\mid a>0,b=0 \}$, $\Gamma_i^{+-}=\{(a,b)\mid a<0,b=0 \}$, $\Gamma_i^{-+}=\{(a,b)\mid a=0,b>0 \}$, $\Gamma_i^{--}=\{(a,b)\mid a=0,b<0 \}.$ In this case of $k_i(\vtheta)=a^2-b^2$, the limit trajectory can also be derived from the  conserved quantity $a(t)b(t)=a(0)b(0)$. Under infinitesimal initialization, we have $a(t)b(t)=0$. This results in the four limit trajectories.

\begin{lemma}[convergence to limit trajectory under infinitesimal initialization]
Consider the general two-layer diagonal linear network in Definition~\ref{def:general 2-layer}. For each $i\in [n]$, let $\vH_i=\nabla^2 k_i(\vtheta_i)$. Let $\lambda_i^+,\vv_i^+$ be the largest eigenvalue and corresponding eigenvector. Let $\lambda_i^-,\vv_i^-$ be the smallest eigenvalue and corresponding eigenvector. Fix $\vtheta=(\vtheta_i)_{i=1}^n$. Assume that $\langle \vtheta_i,\vv_i^+\rangle\neq 0, \langle \vtheta_i,\vv_i^-\rangle \neq 0$. Let $\Gamma^s$ be the trajectory of gradient flow under initialization $s\vtheta$. If there exists $l\in [n]$, $\mu \neq 0$,
$\vtheta^s=(\vtheta_i^s)_{i=1}^n \in \Gamma^s$ such that $\lim_{s\to 0}k_l(\vtheta_l^s)=\mu$, then  $\vtheta_l^s$ converges to a point in $\Gamma^{\mathrm{sign}(\mu),\mathrm{sign}(\langle\vtheta_i,\vv_i^{\mathrm{sign}(\mu)}\rangle)}$.

\label{lemma:convergece to limit trajectory}

\end{lemma}

\begin{proof}
For simplicity, we assume $\mu>0$, $\langle \vtheta_i,\vv_i^+\rangle> 0$. The proofs of other cases  are similar. Under this assumption, $\Gamma^{\mathrm{sign}(\mu),\mathrm{sign}(\langle\vtheta_i,\vv_i^{\mathrm{sign}(\mu)}\rangle)}=\Gamma^{++}$. 
By calculation, the gradient flow of general two-layer diagonal linear network is  
$$\frac{\rd \vtheta_l^s}{\rd t}=\nabla k_l(\vtheta_l^s) \cdot u(\vtheta^s), \vtheta^s_l(0)=s\vtheta_l.$$
Here, $u(\vtheta^s)$ is a scalar. Let $\phi(\vtheta',t)$ be the solution of 
\begin{equation}
\frac{\rd \vtheta_l}{\rd t}=\nabla k_l(\vtheta_l), \vtheta(0)=\vtheta'.  
\label{eq:one neuron SIM}
\end{equation}
Let $\vtheta^s=(\vtheta_i^s)_{i=1}^n \in \Gamma^s$, and assume  $k_l(\vtheta_l^s)=\delta$. 
Then $\vtheta^s_l$ is on the trajectory of Equation~\eqref{eq:one neuron SIM}. Define $\vtheta^s_{\mathrm{ref}}=\phi(s\vtheta_l,\frac{1}{\lambda_{1}}\log \frac{1}{s})$. Since $\vtheta^s_l$ is on the trajectory of Equation~\eqref{eq:one neuron SIM},
there exists $t(s)$ such that $\vtheta^s_l=\phi( \vtheta^s_{\mathrm{ref}},t(s)).$
By Lemma~\ref{lemma:limit solution}, the limit $\vtheta_{\mathrm{ref}}:=\lim_{s\to 0} \vtheta^s_{\mathrm{ref}}$ exists, and $\vtheta_{\mathrm{ref}}\neq \vzero.$ 

By assumption of general diagonal linear network, on  $\Gamma^{++}$, the value of   $k_l(\cdot)$ is unbounded. It is readily verifiable that on $\Gamma^{++}$, $k_l(\cdot)$ is positive and strictly increasing. Moreover, we have  $\mu>0$. Therefore, there exists $\vtheta_l^*\in \Gamma^{++}$ such that $k_l(\vtheta^*_l)=\mu$. Since $\vtheta^*_l \in \Gamma^{++}$, there exists $T\in \sR$ such that $\vtheta^*_l=\phi(\vtheta_{\mathrm{ref}},T).$
Since 
$\vtheta_{\mathrm{ref}}=\lim_{s\to 0} \vtheta^s_{\mathrm{ref}}$, we have $\lim_{s\to 0} \phi(\vtheta^s_{\mathrm{ref}},T)=\vtheta^*_l.$ Therefore, for any $\epsilon>0$, there exists $\delta>0$ such that for all $0<s<\delta$, we have $\| \phi(\vtheta^s_{\mathrm{ref}},T)-\vtheta^*_l\|_2<\epsilon,$ and $|k_l( \phi(\vtheta^s_{\mathrm{ref}},T))-k_l(\vtheta^*_l)|<\epsilon.$ Denote $\vh^s:=\phi(\vtheta^s_{\mathrm{ref}},T)$.
Since we have $\lim_{s\to 0} k_l(\vtheta_l^s)=k_l(\vtheta^*_l)=\mu$, we may assume that $|k_l(\vtheta^s_l)-\mu|<\epsilon.$ Then we have
$$|k_l(\vtheta_l^s)-k_l(\vh^s)|\leq |k_l(\vtheta^s_l)-\mu|+|k_l(\vh^s)-\mu|<2\epsilon.$$
Since $$\vtheta^s_l=\phi( \vtheta^s_{\mathrm{ref}},t(s)),$$
we have 
$$\vtheta^s_l=\phi(\vh^s,t(s)-T).$$

By calculation, 
$$\left.\frac{\rd k_l(\phi(\vtheta^s_{\mathrm{ref}},t))}{\rd t}\right|_{t=T}= \|\nabla k_l(\vh^s)\|_2^2 >0.$$
Since $\|\vh^s-\vtheta^*_l\|_2<\epsilon$, then we may assume 
$$\frac{\rd k_l(\phi(\vtheta^s_{\mathrm{ref}},t))}{\rd t}>\frac{1}{2}\|\nabla k_l(\vtheta^*_l)\|_2^2 $$
for $t$ near $T$.

Since $$|k_l(\vtheta^s)-k_l(\vh^s)|<2\epsilon,$$

the following holds:
$$|t(s)-T|<\frac{4\epsilon}{\|\nabla k_l(\vtheta_l^*)\|_2^2 }$$
when $\epsilon$ is sufficiently small. Therefore $$\vtheta^s_l=\phi(\vh^s,T), t(s)-T)=\phi(\vtheta^*_l+\fO(\epsilon), \fO(\epsilon)).$$ 
Therefore we have 
$$\|\vtheta_l^s-\vtheta^*_l\|_2=\fO(\epsilon).$$
So $\lim_{s\to 0} \vtheta_l^s=\vtheta^*$. The lemma is proved.

\end{proof}

\begin{lemma}[convergence guaranty if initialized at the limit trajectory]
Consider the general two-layer diagonal linear network in Definition~\ref{def:general 2-layer} with data matrix $\vX\in \sR^{m\times n}$ and $\vy\in \sR^m$.
For each $i\in [n]$, let $\Gamma_i^{++},\Gamma_i^{+-},\Gamma_i^{-+},\Gamma_i^{--}$ be the limit trajectories defined  in Definition~\ref{def:limit trajectory}. Let $\vtheta^*=(\vtheta_i^*)_{i=1}^n$  be the initialization. Consider the  gradient flow of general two-layer diagonal linear network:
\begin{equation}
 \frac{\rd \vtheta}{\rd t}=-\nabla L(\vtheta), \vtheta(0)=\vtheta^*, 
 \label{eq:GF convergence initialzed at limit trajectory}
\end{equation}
where $L(\vtheta)=\|\vX\vk(\vtheta)-\vy\|_2^2$.
Let $\vs=(s_1,\ldots,s_n)$ be a signed vector, i.e. $\vs_i\in \{+,-\}$ for all $i\in [n]$. We define $\Gamma_i^{+}:=\Gamma_i^{++}\cup \Gamma_i^{+-}$, and $\Gamma_i^{-}:=\Gamma_i^{-+}\cup \Gamma_i^{--}$.
Assume that for each $i\in [n]$, we have $\vtheta^*_i\in \Gamma^{s_i}$.  Assume that $\vX$ has full column rank.
Denote the solution of  Equation~\eqref{eq:GF convergence initialzed at limit trajectory} by $\vtheta(t)$. Then $\vk(\vtheta(t))$ converges to the solution of the  following optimization  problem:
\begin{equation}
    \min _{\vk\in \sR^n} \|\vX\vk-\vy\|_2^2 \quad s.t. \quad k_i\geq 0 \text{ if } s_i=+, \quad k_i\leq 0 \text{ if } s_i=-.
\label{eq:optimization probelm, lemma convergence}
\end{equation} 
\label{lemma:convergence guaranty if initialized at the limit trajectory}
\end{lemma}

\begin{proof}

Define $I=\{i\in [n]\mid s_i=+\}$, and $J=[n]\setminus{ I}$.
Without loss of generality, we assume $\vtheta_i^*\in \Gamma^{++}$ for all $i\in I$,  and $\vtheta_i^*\in \Gamma^{-+}$ for all $i\in J$. By Theorem~\ref{thm:SIM}, $\vtheta_i(t)\in \Gamma^{++}$ for all $i\in I$, and $\vtheta_j(t)\in \Gamma^{-+}$ for all $j\in J$. It is readily verifiable that on $\Gamma^{++}_i$, the value of $k_i$ is positive. On $\Gamma^{-+}_i$, the value of $k_i$ is negative.
So $k_i(\vtheta_i(t))> 0,\forall i\in I,t\in \sR$, and  $k_i(\vtheta_j(t))< 0,\forall j\in J,t\in \sR $. 

By assumption, $\vX$ has full column rank. Therefore, for any $\delta>0$, the set $\{\vk\in \sR^n\mid \|\vX\vk-\vy\|_2<\delta\}$ is bounded. Let $\vtheta(t)$ be the solution of Equation~\eqref{eq:GF convergence initialzed at limit trajectory}. Since the loss of gradient flow is decreasing, we have $\|\vX\vk(\vtheta(t))-\vy\|_2\leq \|\vX\vk(\vtheta(0))-\vy\|_2$. So $\vk(\vtheta(t))$ is bounded. So $k_i(\vtheta_i(t))$ is bounded for each $i\in [n]$. It is readily verifiable that on $\Gamma_i^{++}$, $\vk_i$ is positive and monotonically increasing. Besides, by assumption, on $\Gamma_i^{++}$, $k_i(\cdot)$ is unbounded. For all $i\in I$, we have $\vtheta_i(t)\in \Gamma^{++}$.
Therefore, since $k_i(\vtheta_i(t))$ is bounded,   $\vtheta_i(t)$ is also bounded. Similarly, one may prove that for $j\in J$,  $\vtheta_j(t)$ is bounded. 
So $\vtheta(t)$ is bounded.

By assumption, for each $i\in [n]$, $k_i(\vtheta_i)$ is a real analytic function. So $L(\vtheta)$ is a real analytic function. By standard Łojasiewicz's inequality~\cite{lojasiewicz1965ensembles} argument, $\vtheta(t)$ converges to a critical point of $L(\vtheta)$. Let $\vtheta'=(\vtheta_i')_{i=1}^n$ be the critical point that $\vtheta(t)$ converges to. Since $\vtheta'=\lim_{t\to +\infty}\vtheta(t)$, we have 
$k_i(\vtheta'_i(t))\geq 0,\forall i\in I$, and  $k_i(\vtheta'_j(t))\leq 0,\forall j\in J.$

Define $\vu(\vtheta)=\vX^\T(\vy-\vX\vk(\vtheta))$, and  $\vu=\vX^\T(\vy-\vX\vk(\vtheta')).$ Since $\vtheta'$ is a critical point of $L(\vtheta)$, so we have $\nabla L(\vtheta')=0$. This is equivalent to 
$$\nabla k_i(\vtheta') u_i=\vzero, \forall i\in [n].$$
Since for each $i\in [n]$, $\nabla k_i(\cdot)$ has unique critical point at $\vzero$, so $\nabla L(\vtheta')=0$ is equivalent to 
$$\vtheta'=\vzero \text{ or } u_i=0, \forall i\in [n].$$
Let $K=\{i\in [n]\mid \vtheta'_i=\vzero\}.$ For any $i\in I\cup K$, it is readily to verify that $u_i\leq 0$. Otherwise assume $u_i>0$. By calculation,  $\vtheta_i$ follows the  dynamics of
$$\frac{\rd \vtheta_i}{\rd t}=\nabla k_i(\vtheta_i) \cdot u_i(\vtheta).$$
So we have 
$$\frac{\rd k_i(\vtheta_i)}{\rd t}=\|\nabla k_i(\vtheta_i)\|_2^2 \cdot u_i(\vtheta).$$
Therefore, for sufficiently large $t$, the value of $k_i(\vtheta_i)$ increases monotonically. This contradicts the fact that $k_i(\vtheta_i(t))>0$ and $k_i(\vtheta_i(t))\to 0$. So $u_i\leq 0$ for all $i\in I\cup K$. Similarly, we have $u_i\geq 0$ for all $i\in J\cup K$. Together with the condition that $\vtheta'=\vzero \text{ or } u_i=0, \forall i\in [n]$, it is readily verifiable that $\vk(\vtheta')$ satisfies the KKT conditions of optimization problem~\eqref{eq:optimization probelm, lemma convergence}. Moreover, since optimization problem~\eqref{eq:optimization probelm, lemma convergence} is convex and satisfies Slater's condition, $\vk(\vtheta')$ is its solution.    
\end{proof}



\begin{lemma}
Let $\vy\in \sR^m$. Let $U$ be a subspace of $\sR^m$.
Let $\vx^*$ be the solution of the optimization problem:
\[ \min_{\vx\in U}  \|\vx-\vy\|_2^2. \]
Fix $L>0$.
Then for any $\epsilon>0$, there exists $\delta>0$, such that if $\vx \in U$ satisfies $\|\vx-\vy'\|_2 <L\delta+ \|\vx^*-\vy\|_2$ and $\|\vy'-\vy\|_2 < \delta$, then $\|\vx-\vx^*\|_2 < \epsilon$.
\label{lemma:small change}
\end{lemma}

\begin{proof}
Let $d = \|\vx^* - \vy\|_2$. If $d=0$, the statement is trivial. So
we assume $d > 0$.
As $\vx^*$ is the orthogonal projection of $\vy$ onto $U$, the vector $\vx^* - \vy$ is orthogonal to the subspace $U$. Since $\vx - \vx^* \in U$, it follows that $\langle \vx - \vx^*, \vx^* - \vy \rangle = 0$.

By the Pythagorean theorem:
\[ \|\vx - \vy\|_2^2 = \|(\vx - \vx^*) + (\vx^* - \vy)\|_2^2 = \|\vx - \vx^*\|_2^2 + \|\vx^* - \vy\|_2^2. \]
Rearranging this identity, we get an expression for the term we wish to bound:
\[ \|\vx - \vx^*\|_2^2 = \|\vx - \vy\|_2^2 - \|\vx^* - \vy\|_2^2 = \|\vx - \vy\|_2^2 - d^2. \]
We introduce $\vy'$ by writing $\vx - \vy = (\vx - \vy') + (\vy' - \vy)$. Substituting this gives:
\begin{align*}
    \|\vx - \vx^*\|_2^2 &= \|(\vx - \vy') + (\vy' - \vy)\|_2^2 - d^2 \\
    &= \|\vx - \vy'\|_2^2 + 2\langle \vx - \vy', \vy' - \vy \rangle + \|\vy' - \vy\|_2^2 - d^2.
\end{align*}
Using the given conditions $\|\vx - \vy'\|_2 < d$ and $\|\vy' - \vy\|_2 < \delta$:
\[ \|\vx - \vx^*\|_2^2 < (L\delta+d)^2 + 2\langle \vx - \vy', \vy' - \vy \rangle + \delta^2 - d^2 = 2\langle \vx - \vy', \vy' - \vy \rangle + (L^2+1)\delta^2+2dL\delta. \]
By the Cauchy-Schwarz inequality and the given conditions again:
\[ \|\vx - \vx^*\|_2^2 < 2\|\vx - \vy'\|_2 \|\vy' - \vy\|_2 + (L^2+1)\delta^2+2dL\delta < 2d(L+1)\delta+(L+1)^2\delta^2 \]
For a given $\epsilon > 0$, we could pick $\delta > 0$  sufficiently small such that $d(L+1)\delta+(L+1)^2\delta^2 < \epsilon^2$.
 For this choice of $\delta$:
\[ \|\vx - \vx^*\|_2^2 < \epsilon^2. \]
Taking the square root yields $\|\vx - \vx^*\|_2 < \epsilon$.
\end{proof}

\begin{lemma}
Let $\vk(\vtheta)=(k_i(\vtheta_i))_{i=1}^n \in \sR^n$, where $\vtheta=(\vtheta_i)_{i=1}^n$.
Assume that for each $i\in [n]$, $\vk_i$ is a real analytic activation, and $k_i(\vzero)=\nabla k_i(\vzero)=0$. Let $\vX\in \sR^{m\times n}, \vy\in \sR^{m}$. Define $L(\vtheta)=\|\vX \vk(\vtheta)-\vy\|_2^2.$ Fix $\vtheta^*$, and define $\vk^*=\vk(\vtheta^*)$.
Assume that $\vk^*$ solves the following optimization problem for some $I_1,I_2,I_3,\{j\}\subset [n], $:
$$\min \|\vX \vk-\vy\|_2^2, \quad s.t. \quad \vk_i\geq 0,i\in I_1;\vk_i \leq 0,i\in I_3; \vk_i=0,i\in I_2\setminus{\{j\}}. $$
Define $I=\{i\in [n]\mid \vk_i^{*}=0\}$. Then
for any $\epsilon>0$, there exists $\delta>0$ such that the following holds:
``For all initialization $\vtheta'$ such that $\|\vk(\vtheta')-\vk^*\|_2<\delta$, let $\vtheta(t)$ be the solution of $\frac{\rd \vtheta}{\rd t}=-\nabla L(\vtheta), \vtheta(0)=\vtheta'$. Then 
 $\|\vX\vk(\vtheta(t))-\vX\vk(\vtheta^*)\|_{\infty}<\epsilon$ for all $0\leq t<T$, where $T=\inf_{t\geq 0} \{t\mid \|\vk_I(\vtheta(t))\|_{\infty}>\delta\}$.
''
\label{lemma:small change of large neurons}
\end{lemma}

\begin{proof}
Define $J=[n]\setminus{I}$.
Define $U=\mathrm{span}( \{\vX_{:,i}\mid i\in J\}).$
 Since $\vk^{*}$ solves the optimization problem
$$\min \|\vX \vk-\vy\|_2^2, \quad s.t. \quad \vk_i\geq 0,i\in I_1;\vk_i \leq 0,i\in I_3; \vk_i=0,i\in I_2\setminus{j}, $$
So $\vk_J^{*}$ must solve the following problem:
$$\min \|\vX_{:,J} \vk_J-\vy\|_2^2.$$
Let $\vz^*=\vX_{:,J} \vk_J\in U$. Then  $\vz$ is the solution of the optimization problem
$$\min_{\vz\in U} \|\vz-\vy\|_2^2.$$

 Since the loss of gradient flow is decreasing, so 
 $$
 \|\vX\vk(\vtheta(t))-\vy\|_2<\|\vX\vk(\vtheta')-\vy\|_2 \leq \|\vX\vk(\vtheta')-\vX\vk(\vtheta^*)\|_2+\|\vX\vk(\vtheta^*)-\vy\|_2 \leq L\|\vk(\vtheta')-\vk^*\|_2 +\|\vz^*-\vy\|_2.$$
 Here, $L$ is  $\|\vX\|_2$. By Lemma~\ref{lemma:small change}, for any $\epsilon>0$, there exists $\delta>0$ such that if $\vz\in U$ satisfies $\|\vz-\vy'\|_2<L\delta+\|\vz^*-\vy\|_2$ and $\|\vy'-\vy\|<\delta$, then $\|\vz-\vz^*\|<\epsilon.$
Assume 
 $\|\vk(\vtheta')-\vk^*\|<\delta'$ for some $\delta'>0$ to be determined.  Then we have 
$$\|\vX\vk(\vtheta(t))-\vy\|_2<L\delta'+\|\vz^*-\vy\|_2.$$
We may pick $\delta'<\delta$, so that 
$$\|\vO(\vtheta(t))-\vy\|_2<L\delta+\|\vz^*-\vy\|_2.$$
Besides,  we have 
$$\vX\vk(\vtheta(t))=\sum_{i\in I}k_i(\vtheta_i(t))\vX_{:,i}+\sum_{i\in J}k_i(\vtheta_i(t))\vX_{:,i}.$$


We may pick $\delta'$ sufficiently small such that 
$$\|\vk_I(\vtheta(t))\|_2<\delta' \Longrightarrow \|\sum_{i\in I}k_i(\vtheta_i(t))\vX_{:,i}\|_2<\delta.$$
So as long as $\|\vk_I(\vtheta(t))\|_2<\delta'$, we have 
$$\|\sum_{i\in J}k_i(\vtheta_i(t))\vX_{:,i}-(\vy-\sum_{i\in I}k_i(\vtheta_i(t))\vX_{:,i})\|_2<L\delta+\|\vz^*-\vy\|_2.$$
Besides, we have $\|\sum_{i\in I}k_i(\vtheta_i(t))\vX_{:,i})\|_2<\delta$, and $\sum_{i\in J}k_i(\vtheta_i(t))\vX_{:,i}\in U.$
By Lemma~\ref{lemma:small change}, we have $$\|\sum_{i\in J}k_i(\vtheta_i(t))\vX_{:,i}-\vz^*\|_2<\epsilon.$$ So we have 
$$\|\vX\vk(\vtheta(t))-\vX\vk^*\|_{2} \leq \|\sum_{i\in J}k_i(\vtheta_i(t))\vX_{:,i}-\vz^*\|_2+\|\sum_{i\in I}k_i(\vtheta_i(t))\vX_{:,i}\|_2<\epsilon+\delta'.$$
We may pick $\delta'<\epsilon$, then we have 
$$\|\vX\vk(\vtheta(t))-\vX\vk^*\|_{2}<2\epsilon.$$
Since in $L_2$ norm is equivalent to $L_{\infty}$ norm, the lemma holds.

\end{proof}

\begin{definition}[time mapping]
Consider the general two-layer diagonal linear network, and let  Assumption~\ref{assum:L-layer} holds. For each $i\in [n]$, let $\vtheta_i^s(t)$ be the solution of 
$$\frac{\rd \vtheta_i}{\rd t}=\nabla k_i(\vtheta_i),\vtheta_i(0)=s\vtheta_i^*.$$
Define $\Gamma_i^s=\{\vtheta_i^s(t)\mid t\in \sR\}.$
It is readily verifiable that for any $\vtheta_i\in \Gamma_i(\vtheta_i^*)$,  there exists unique $t\in \sR$ such that $\vtheta_i=\vtheta_i^s(t)$. The map from $\vtheta_i$ to $\frac{t}{-\log s}$ is denoted by $\frac{t}{-\log s}=\tau^s_i(\vtheta_i)$.
We call $\tau^s_i(\vtheta_i)$ as the  time mapping of neuron $i$. The time mapping $\tau^s_i(\vtheta_i)$ makes sense only if $\vtheta_i\in \Gamma_i^s$. Besides, we define $\vtau^s(\vtheta)=(\tau_i^s(\vtheta_i))_{i=1}^n.$ We call $\vtau^s$ as the time mapping.

\label{def:time mapping}
\end{definition}

\begin{lemma}
Consider the notation and assumptions in Definition~\ref{def:time mapping}. Fix $i\in [n]$.
Let $\vH=\nabla^2 k_i(\vzero)$. Assume $\vH$ has a unique largest eigenvalue $\lambda_{1}>0$ with corresponding eigenvectors $\vv_{1}$. Assume $\vH$ has a unique smallest eigenvalue $\lambda_{2}<0$ with the corresponding eigenvectors $\vv_{2}$. Assume $\langle\vtheta_i^*,\vv_1\rangle, \langle\vtheta_i^*,\vv_2\rangle \neq 0$. 
Then the following holds:
\begin{itemize}
    \item Fix $0<\delta<1$.
    Pick $\vtheta_i^s\in \Gamma_i^s(\vtheta_i^*) $ for each $0<s<\delta$. Assume there exists $\epsilon>0$ such that
    $\frac{1}{\lambda_2}+\epsilon<\tau_i^s(\vtheta_i^s)<\frac{1}{\lambda_1}-\epsilon$ for all $0<s<\delta.$ Then $\lim_{s\to 0} \vtheta_i^s=\vzero.$

    \item Fix $0<\delta<1$.
    Pick $\vtheta_i^s\in \Gamma_i^s(\vtheta_i^*) $ for each $0<s<\delta$. Assume there exists $\epsilon>0$ such that
    $\tau_i^s(\vtheta_i^s)>\frac{1}{\lambda_1}+\epsilon$ for all $0<s<\delta.$ Then $\lim_{s\to 0} k_i(\vtheta_i^s)=+\infty.$

\item Fix $0<\delta<1$.
    Pick $\vtheta_i^s\in \Gamma_i^s(\vtheta_i^*) $ for each $0<s<\delta$. Assume there exists $\epsilon>0$ such that
    $\tau_i^s(\vtheta_i^s)<\frac{1}{\lambda_2}-\epsilon$ for all $0<s<\delta.$ Then $\lim_{s\to 0} k_i(\vtheta_i^s)=-\infty.$


    

\label{lemma: short time leads to zero}    
\end{itemize}

\label{lemma:time mapping properties}
\end{lemma}

    

\begin{proof}


This lemma is a direct corollary of Lemma~\ref{lemma:limit solution}. We  leave out the  proof.
\end{proof}

\begin{theorem}[Dynamics of General Two-layer Diagonal Linear Network]
Consider the general two-layer diagonal linear network, and let  Assumption~\ref{assum:L-layer} holds. Let $\Gamma^s$ to be the trajectory of GF initialized at $s\vtheta^*$. Let $\vtau^s$ be the time mapping defined in Definition~\ref{def:time mapping}.
Then the following holds:
\begin{itemize}
    \item For each $p=0,1,\ldots,p_{\max}$, there exists $\vtheta^s\in \Gamma^s$ such that  $\lim_{s\to 0}\vk(\vtheta^s) =\vk^{(p)}$ and $\lim_{s\to 0}\vtau^s(\vtheta^s) =\vs^{(p)}.$
    \item The limit $\vtheta'=\lim_{s\to +\infty}\lim_{t\to +\infty} \phi(s\vtheta^*,t)$ exists, and $\vk(\vtheta')=\vk^{(p_{\max})}$.
\end{itemize}

\label{thm:L=2}
\end{theorem}
\begin{proof}

For $p=0$, the statements obviously hold. Assume the statements hold for some natural number $p$. We want to prove that the statements  hold for $p+1$.  


Define $I=\{i\in [n]\mid k_i^{(p)}=0\}$,  $J_1=\{i\in [n]\mid k_i^{(p)}>0\}$, , $J_2=\{i\in [n]\mid k_i^{(p)}<0\}$.
Define $\vu(\vtheta)=\vX^\T(\vy-\vX \vk(\vtheta))$. Let $\vtheta^s(t)$ be the solution of 
$$\frac{\rd \vtheta }{\rd t}=-\nabla L(\vtheta), \vtheta(0)=\vtheta^s.$$


By calculation, we have 
$$\frac{\rd \vtheta^s_i(t)}{\rd t}=\nabla k_i(\vtheta^s_i(t)) u_i(\vtheta^s(t)).$$
For each $i\in [n]$, let $T_i^s(t)$ to be the solution of 
$$\frac{\rd T_i^s(t)}{\rd t}=u_i(\vtheta^s(t)), T_i^s(0)=0.$$
Then we have 
$$\frac{\rd \vtheta^s_i(T_i^s)}{\rd T_i^s}=\nabla k_i(\vtheta^s_i(T_i^s)), \vtheta_i^s(0)=\vtheta^s_i.$$
Let $\vtau^s$ be the time mapping defined in Definition~\ref{def:time mapping}. By Theorem~\ref{thm:SIM}, $\tau^s_i(\vtheta_i^s(t))$ is well defined for each $i\in [n]$.
By definition of time mapping, for all $i\in [n]$ and $s>0$, we have 
\begin{equation}
\tau_i^s(\vtheta_i^s(t))=\frac{T_i^s(t)}{-\log s}+\tau_i^s(\vtheta^s_i).   
\label{eq: time equation}
\end{equation}

Let $\vu^*:=\vX^\T(\vy-\vX \vk^{(p)}
)$. By Lemma~\ref{lemma:small change of large neurons}, for any $\epsilon>0$, there exists $\delta>0$ such that for sufficiently large $s$, we have 
 $\|\vu(\vtheta^s(t))-\vu^*\|_{\infty}<\epsilon$, where $T^s=\inf_{t\geq 0} \{t\mid \|\vk_{I}(\vtheta(t))\|_{\infty}>\delta\}$. 

For $0\leq t<T^s$, we have 
$$u_i^*-\epsilon <\frac{ \rd T_i^s(t)}{\rd t}<u_i^*+\epsilon.$$
So for $0\leq t<T^s$, it holds that 
$$(u_i^*-\epsilon)t < T_i^s(t)<(u_i^*+\epsilon)t.$$
By Equation~\eqref{eq: time equation} and the assumption that $\lim_{s\to 0} \vtau^s(\vtheta^s)=\vs^{(p)}$, for sufficiently small $s$, we have 
$$(u_i^*-\epsilon)\frac{t}{-\log s}+(s^{(p)}_i-\epsilon)<\tau_i^s(\vtheta_i^s(t))<(u_i^*+\epsilon)\frac{t}{-\log s}+(s^{(p)}_i+\epsilon),\quad \forall 0<t<T^s.$$

For simplicity of notation, we assume $u_i^*>0$ for each $i\in I$. The case of  $u_i^*<0$ for some $i$ is similar.
For each $i\in I$, define $\mu_i(\epsilon)$ to be the  interval $[\frac{t_i^+-(s_i^{(p)}+\epsilon)}{u_i^*+\epsilon},\frac{t_i^+-(s_i^{(p)}-\epsilon)}{u_i^*-\epsilon}].$ We use $\mu_i^-(\epsilon)$ and $\mu_i^+(\epsilon)$ to denote the left endpoint and right endpoint of $\mu(\epsilon)$, respectively.

Fix $i\in I,0<a<\mu_i^-(\epsilon)$. 
By Lemma~\ref{lemma: short time leads to zero}, if  $0<t^s<\min \{a\log\frac{1}{s} ,T^s\}$ for each $s>0$, then $\lim_{s\to 0} \vtheta_i^s(t)=\vzero.$
Besides, if $a>\mu_i^+(\epsilon)$ and $0<t^s<\min \{a\log\frac{1}{s} ,T^s\}$ for each $s>0$, then $\lim_{s\to 0} k_i(\vtheta_i^s(t))=+\infty.$

Since $T^s=\inf_{t\geq 0} \{t\mid \|\vk_{I}(\vtheta(t))\|_{\infty}>\delta\}$, we have $\|\vk_{I}(\vtheta(T^s))\|_{\infty}=\delta.$ Therefore, there exists $l\in I$, such that $$k_l(\vtheta_l(T^s))=\pm \delta.$$
Since $u_i^*>0$ for each $i$, it holds that 
$$k_l(\vtheta_l(T^s))= \delta.$$

For each  \( i \in I \), we define   \( \mu_i(0)=(t_i^+ - s_i^{(p)}) / u_i^*. \) Let $j \in \mathrm{argmin}_{i \in I_2} \mu_i(0) $. 
By assumption, $j$ is unique.
Therefore, for sufficiently small $\epsilon$, the  intervals $\mu_i(\epsilon), i\in [n]$ has no intersection.

It is readily verifiable that for sufficiently small $\epsilon$, we have $l=j$. Otherwise let $l\neq j$. By Lemma~\ref{lemma: short time leads to zero}, for any $\hat{\epsilon}>0$, if $T^s/\log\frac{1}{s}< \mu_l^-(\epsilon)-\hat{\epsilon}$ holds for sequence $s_n\to 0$, then $k_l(\vtheta_l(T^s))\to 0$, which contradicts $k_l(\vtheta_l(T^s))=\delta>0$. So for sufficiently small $s$, we have 
$$T^s/\log\frac{1}{s}> \mu_l^-(\epsilon)-\hat{\epsilon}.$$ 

We may pick $\hat{\epsilon}$ sufficiently small such that 
$$\mu_l^-(\epsilon)-\hat{\epsilon}>\mu_j^+(\epsilon).$$

Therefore, we have 
$$T^s/\log\frac{1}{s}> \mu_l^-(\epsilon)-\hat{\epsilon}>\mu_j^+(\epsilon).$$

By Lemma~\ref{lemma: short time leads to zero}, we have
$k_j(\vtheta_j^s(T^s)) \to +\infty$, which contradicts that $|k_j(\vtheta_j^s(T^s)) |\leq \delta.$

Moreover, we have $T^s/\log\frac{1}{s}\in \mu_j(\epsilon).$ Therefore, for all $i\in I$, it holds that 
$$\tau_i^s(\vtheta_i^s(T^s))=u_i^*\mu_j(0)+s^{(p)}_i+\fO(\epsilon).$$
As a consequence, for sufficiently small $\epsilon$, for each $i\in I\setminus{\{j\}}$, we have 
$$\lim_{s\to 0} \vtheta_i^s(T^s)=\vzero.$$
By Lemma~\ref{lemma:convergece to limit trajectory}, $\vtheta_j^s(T^s)$ converges to one point in $\Gamma^{++}$ or $\Gamma^{+-}$. Define $J=J_1\cup J_2$.
For $i\in J$, we already have the control $\|\vX_{:,J}\vk_J(\vtheta_J^s(T^s))-\vX_{:,J}\vk_J^{(p)}\|_{\infty}<\epsilon.$ By Lemma~\ref{lemma: linearly indpendent}, $\vX_{:,J}$ has full column rank. So we have 
$$\|\vk_J(\vtheta_J^s(T^s))-\vk_J^{(p)}\|_{\infty}<C\epsilon,$$
where $C$ depends only on $\vX_{:,J}$. By Lemma~\ref{lemma:convergece to limit trajectory}, for each $i\in J$, the limit $\lim_{s\to 0} \vtheta^s_i(T^s)$ exists, and $\lim_{s\to 0} \vtheta^s_i(T^s)$ is on the limit trajectories in Definition~\ref{def:limit trajectory}.

Therefore, for all $i\in [n]$, the limit $\vtheta_i^s(T^s)$ exists. So $$\hat{\vtheta}_0:=\lim_{s\to 0}\vtheta^s(T^s)$$ exists. Let $\hat{\vtheta}(t)=\phi(\hat{\vtheta}_0,t).$
By Theorem~\ref{thm:SIM}, for each $i\in I\setminus{\{j\}}$, for all $t>0$, we have  $\hat{\vtheta}_i(t)=\vzero.$ Then apply Lemma~\ref{lemma:convergence guaranty if initialized at the limit trajectory}, the limit 
$\vk^*=\lim_{t\to \infty} \vk(\hat{\vtheta}(t))$ exists, and solves the optimization problem:
\begin{equation}
    \min _{\vk\in \sR^n} \|\vX\vk-\vy\|_2^2 \quad s.t. \quad k_i\geq 0 \text{ if } s_i=+, \quad k_i\leq 0 \text{ if } s_i=-.
\end{equation} 
By definition, $\vk^*=\vk^{(p+1)}$. So for any $\epsilon'>0$, there exists $T>0$ such that 
$$\|\vk^*- \vk(\hat{\vtheta}(T))\|_2<\epsilon'.$$

Since $\hat{\vtheta}_0:=\lim_{s\to 0}\vtheta^s(T^s)$, for sufficiently small $s$, we have 
$$\|\phi(\vtheta^s(T^s),T)-\hat{\vtheta}(T)\|_2<\epsilon',$$
and 
$$\|\vk(\phi(\vtheta^s(T^s),T))-\vk(\hat{\vtheta}(T))\|_2<\epsilon'.$$
Let $\vtheta_s:=\phi(\vtheta^s(T^s),T)$. Then we have 
$$\|\vk(\vtheta_s)-\vk^{(p+1)}\|_2<2\epsilon'.$$
Therefore, we have 
$$\lim_{s\to 0}\vk(\vtheta_s)=\vk^{(p+1)}.$$
By definition, 
$$\vtau^s(\vtheta_s)=\vtau^s(\vtheta^s(T^s))+\frac{T}{-\log s }.$$
Let $s\to 0$, one gets
$$\lim_{s\to 0}\vtau^s(\vtheta_s)=\lim_{s\to 0}\vtau^s(\vtheta^s(T^s))=\vu^* \mu_j(0)+\vs^{(p)}.$$
By definition, 
$\vu^* \mu_j(0)+\vs^{(p)}=\vs^{(p+1)}.$
So 
$$\lim_{s\to 0}\vtau^s(\vtheta_s)=\vs^{(p+1)}.$$
 So the statements of induction  hold for $p+1$. By mathematical induction, the statements hold for any natural number  $p$.

In the following we prove that the limit $\vtheta'=\lim_{s\to +\infty}\lim_{t\to +\infty} \phi(s\vtheta^*,t)$ exists, and $\vk(\vtheta')=\vk^{(p_{\max})}$. We have already proved that there exists $\vtheta_s$ such that $\vk(\vtheta_s)\to \vk^{p_{\max}}$. By Lemma~\ref{lemma:convergece to limit trajectory}, the limit $\tilde{\vtheta}:=\lim_{s\to 0}\vtheta_s$ also exists.

By definition of $\vk^{p_{\max}}$, $L(\tilde{\vtheta})=0$. So $\tilde{\vtheta}$ is a global minimizer of $L(\vtheta)$. Since $L(\vtheta)$ is real analytic, Łojasiewicz inequality~\citet{lojasiewicz1965ensembles} holds for $L(\vtheta)$.

By standard Łojasiewicz inequality argument (for example, see Lemma G.1 in~\citet{li2020towards}) , there exists  $C,\alpha,\delta>0$, such that for any $\vtheta_0$ in $B_{\delta}(\tilde{\vtheta})$, the limit $\lim_{t\to \infty}\phi(\vtheta_0,t)$ exists, and $$\|\lim_{t\to \infty}\phi(\vtheta_0,t)-\tilde{\vtheta}\|_2<C|\vtheta_0-\hat{\vtheta}|_2^\alpha.$$ 

Since 
$\lim_{s\to +\infty}\vtheta_s =\tilde{\vtheta}$, for sufficiently small $s$, we have $$\|\lim_{t\to \infty}\phi(\vtheta_s,t)-\hat{\vtheta}\|_2<C\|\vtheta_s-\hat{\vtheta}\|_2^\alpha.$$
Let $s\to 0$, we get 
$$\lim_{s\to 0}\lim_{t\to +\infty} \phi(s\vtheta^*,t)=\tilde{\vtheta}.$$

\end{proof}

\begin{theorem}
Consider the deep diagonal linear network, and let Assumption~\ref{assum:L-layer} holds. Let $\Gamma_s$ to be the trajectory of GF initialized at $s\vtheta^*$. We use $a_{i,j}(\vtheta)$ to denote the value of $a_{i,j}$ at $\vtheta$. Then the following holds:

\begin{itemize}
    \item For each $p=0,1,\ldots,p_{\max}$, there exists $\vtheta^s\in \Gamma^s$ such that  
    \begin{itemize}
        \item $\lim_{s\to 0}\vk(\vtheta^s) =\vk^{(p)}$, and
        \item $\lim_{s\to +\infty} \frac{a_{i,L}(\vtheta_s)}{s^{L-2}} =F_k^{-1}(s_k^{(p)}).$
    \end{itemize}
    
    \item The limit $\vtheta'=\lim_{s\to +\infty}\lim_{t\to +\infty} \phi(s\vtheta^*,t)$ exists, and $\vk(\vtheta')=\vk^{(p_{\max})}$.
\end{itemize}

\label{thm:L-layer}
\end{theorem}

\begin{proof}
The main idea of proof is the same with Theorem~\ref{thm:L=2}. 
By calculation, we have 
\[
\frac{\rd a_{i,\ell}}{\rd t}
=
\Big(\prod_{r\neq \ell} a_{i,r}\Big)\, u_i(\vtheta),
\qquad
 \vu(\vtheta)=\vX^\T(\vy-\vX\vk(\vtheta)).
\]
For any $i=1,\ldots,n$, $j=1,\ldots,L-1$,
we have the preserved quantity
\[
a_{i,j}^2(t)-a_{i,L}^2(t)
=
a_{i,j}^2(0)-a_{i,L}^2(0).
\]
By assumption, $a_{i,j}^2(0)-a_{i,L}^2(0)=s^2\mu_{i,j}^2$.
So we have 
\[
a_{i,j}^2(t)-a_{i,L}^2(t)
=s^2\mu_{ij}^2>0.
\]
Therefore, for any $t\geq 0$, the value of $a_{i,j}^2(t)$ can not be zero. So for 
$i=1,\ldots,n$, $j=1,\ldots,L-1$,
the sign of $a_{i,j}(t)$ does not change during training. As a consequence, we have 
\[
\frac{\rd a_{i,L}}{\rd t}
= \pm
\Big(\prod_{r\neq L} \sqrt{a_{i,L}^2+s^2\mu^2_{i,r}}\Big)\, u_i(\vtheta),\]
The sign of the $\pm$ is same with the sign of $\prod_{r\neq L} a_{i,r}(0)$. Without loss of generality we assume that  $\prod_{r\neq L} a_{i,r}^*>0$ for all $i$. Otherwise one can replace $a_{i,L}(t)$ with $-a_{i,L}(t)$. Under the assumption, we have 
\begin{equation}
\frac{\rd a_{i,L}}{\rd t}
= 
\Big(\prod_{r\neq L} \sqrt{a_{i,L}^2+s^2\mu^2_{i,r}}\Big)\, u_i(\vtheta).
\label{eq:GF of L layer}
\end{equation}
For $p=0$, it is readily verifiable that the statements of the theorem hold.
Now assume the statements of the theorem hold for natural number $p$. We want to prove the statements for $p+1$. 

Define $I_1=\{i\in [n]\mid k_i^{(p)}>0\}$, $I_2=\{i\in [n]\mid k_i^{(p)}=0\}$, $I_3=\{i\in [n]\mid k_i^{(p)}<0\}$. 
 Denote $\vtheta^s(t)=\phi(\vtheta^s,t)$, and denote $a_{ij}^s(t)$ to be the value of $a_{ij}$ of $\vtheta^s(t)$.


 Define $\vu(\vtheta)=\vX^\T(\vy-\vX \vk(\vtheta))$, and define $\vu^*=\vX^\T(\vy-\vX \vk^{(p)})$. 
By Lemma~\ref{lemma:small change of large neurons}, for any $\epsilon>0$, there exists $\delta>0$ such that for sufficiently small $s$, we have 
 $$\|\vu(\vtheta^s(t))-\vu^*\|_{\infty}<\epsilon$$ for all $0\leq t\leq T^s$, where $T^s=\inf_{t\geq 0} \{t\mid \|\vk_{I_2}(\vtheta^s(t))\|_{\infty}>\delta\}$. 
 

 By Equation~\eqref{eq:GF of L layer}, we have 
$$\frac{\rd a^s_{i,L}}{\prod_{r\neq L} \sqrt{(a_{i,L}^{s})^2+s^2\mu^2_{i,r}}}
= 
\, u_i(\vtheta^s(t)) \rd t.$$
Take the integration, and we get
$$\int_{a^s_{i,L}(0)}^{a^s_{i,L}(t)}  \frac{\rd x}{\prod_{r\neq L} \sqrt{x^2+s^2\mu^2_{i,r}}}
= \int _{0}^t 
\, u_i(\vtheta^s(r)) \rd r.$$
Change variables by defining $x=sx'$. Then 
$$\int_{ a_{i,L}(0)}^{a^s_{i,L}(t)}  \frac{\rd x}{\prod_{r\neq L} \sqrt{x^2+s^2\mu^2_{i,r}}}
=\int _{a_{i,L}(0)/s}^{a^s_{i,L}(t)/s} s^{2-L} \frac{\rd x'}{\prod_{r\neq L} \sqrt{(x')^2+\mu^2_{i,r}}} .$$
Therefore, we have 
$$\int _{a_{i,L}(0)/s}^{a^s_{i,L}(t)/s} s^{2-L} \frac{\rd x}{\prod_{r\neq L} \sqrt{x^2+\mu^2_{i,r}}}=\int _{0}^t 
\, u_i(\vtheta^s(r)) \rd r.$$
By the definition of $F_i(z)$, we have 
$$F_i(\frac{a^s_{i,L}(t)}{s^{L-2}})-F_i(\frac{a^s_{i,L}(0)}{s^{L-2}})=\int _{0}^t 
\, u_i(\vtheta^s(r)) \rd r.$$
Since $\|\vu(\vtheta^s(t))-\vu^*\|_{\infty}<\epsilon$ for all $0\leq t\leq T^s$, so we have 
\begin{equation}
(u_i^*-\epsilon)t  <F_i(\frac{a^s_{i,L}(t)}{s})-F_i(\frac{a^s_{i,L}(0)}{s})<(u_i^*+\epsilon)t, \forall t\in [0,T^s].
\label{eq:time estimate of L-layer}
\end{equation}
Since $\|\vk_{I_2}(\vtheta^s(T^s))\|_{\infty}=\delta$, then there exists $j\in I_2$ such that 
$$k_j(\vtheta^s(T^s))=\pm \delta.$$
Without loss of generality, let us assume that $u^*_j>0$. Then $k_j(\vtheta^s(T^s))>0.$ So we have 
$$k_j(\vtheta^s(T^s))= \delta.$$

Since we have the conserved quantities
\[
(a_{j,l}^s)^2(t)-(a_{j,L}^s)^2(t)
=s^2\mu_{jl}^2\quad , \quad  \forall l\in [L]
\]
So we have 
$$\lim_{s\to 0} a_{j,L}^s=\delta^{1/L}.$$

Apply Equation~\eqref{eq:time estimate of L-layer} to neuron $j$ and  $t=T^s$,  we get 
$$T^s=\frac{t_j^+-F_j^{-1}(\vs_j^{(p)})}{u_j^*}+h(\epsilon,s)$$
Here, $h(\epsilon,s)\to 0$ as $(\epsilon,s)\to \vzero$.
For each $i\in I_2$, define
\[
T_i = 
\begin{cases}
\frac{t_i^+ - F_i^{-1}(s_i^{(p)})}{|u_i|}, & \text{if } u_i > 0 \\
\frac{t_i^- + F_i^{-1}(s_i^{(p)})}{|u_i|}, & \text{if } u_i < 0 \\
+\infty, & \text{if } u_i = 0
\end{cases}
\]

Let $l\in \mathrm{argmin} _{k\in I_2} T_k$.
By assumption, $l$ is unique. Therefore 
$T_{l}$ is strictly smaller than $T_k$ if $k\neq l$. 

It is readily verifiable that $j=l$ for sufficiently small $\epsilon$ and $s$, otherwise $a_{j,L}(t)$ will first reach $\delta$. 

For $i\in I_2\setminus{j}$, by Equation~\eqref{eq:time estimate of L-layer}, we have 
$$(u_i-\epsilon)T^s  <F_i(\frac{a^s_{i,L}(T^s)}{s})-F_i(\frac{a^s_{i,L}(0)}{s})<(u_i+\epsilon)T^s.$$
So we have 
$$F_i(\frac{a^s_{i,L}(T^s)}{s})=F_i(s_i^{(p)})+u_iT_j+o(1).$$
Here, $o(1)\to 0$ as $\epsilon,s\to 0$.
Since $F_i(s_i^{(p)})+u_iT_j \in (-t_i^-,t_i^+)$ for all $i\in I_2\setminus{\{j\}}$, then  
$a_{i,L}^s(T^s)=O(s)$. So we have 
$$\lim_{s\to 0}a_{i,L}^s(T^s)=0,\forall i\in I_2\setminus{\{j\}}.$$
Similar to the proof of Theorem~\ref{thm:L=2}, the limit $\vtheta'=\lim_{s\to 0} \vtheta^s(T^s)$ exists. Besides, $a_{i,l}(\vtheta^s)=0,\forall i\in I_2\setminus{\{j\}},l\in [L].$
The following part of proof is the same with Theorem~\ref{thm:L=2}, we leave out the details.
\end{proof}

\section{Concepts in Differential Geometry}
\label{sec:diffenrential geometry}
In the appendix, we present several definitions and concepts in  differential geometry that are pertinent to the content of this paper.

\begin{definition}[analytic manifold, page 3 and 4 of~\citet{jurdjevic1997geometric})]
$\fM$ is called an $n$ dimensional analytic manifold if $\fM$ is a topology space such that at each point $p\in \fM$
 there exists a neighbourhood $U$ of $p$ and a homeomorphism $\phi$ from $U$ onto an open subset of $\sR^n$.  It is assumed that $n$ does not vary with the choice of a point $p$ on $\fM$. The pair $(\phi,U)$ is called a chart at $p$. Moreover:
 \begin{enumerate}
 \item There exists a countable collection of charts $\{(\phi_i,U_i)\}_{i=1}^{\infty}$ such that $\fM=\bigcup_{i=1}^{\infty} U_i$.
     \item For each pair of points $p_1$ and $p_2$, there exist charts $(\phi_1,U_1)$ and $(\phi_2,U_2)$ such that $p_1\in U_1$, $p_2\in U_2$, and $U_1\cap U_2=\emptyset$. That is, points of $\fM$ are separated by coordinate neighborhoods (i.e., $\fM$ is Hausdorff).
        \item  For any charts $(\phi_1,U_1)$ and $(\phi_2,U_2)$ such that $U_1\cap U_2\neq \emptyset$, the mapping $\phi_1\circ \phi_2^{-1}$ is analytic as a mapping from an open set in $\sR^n$ into $\sR^n$.

 \end{enumerate}
\end{definition}

\begin{definition}[analytic vector fields, Definition 1 in Chapter 1 of~\citet{jurdjevic1997geometric}] Let $\fM$ be an analytic manifold. The totality of $(p,v),p\in \fM, v\in T_p \fM$, is called the tangent bundle of $\fM$ and is denoted by $T\fM$. A vector field is a mapping $X: \fM \to T\fM$ such that for each $p \in \fM$, if $\pi: T\fM \to \fM$ denotes the natural projection, then $\pi\big(X(p)\big)=p$. We say that $X$ is an analytic vector field if $X$ is an analytic map from $\fM$ (as an analytic manifold) into $T\fM$ (another analytic manifold).
    
\end{definition}

\begin{definition}[Invariant manifold]
Let $\fM$ be an immersed submanifold of $\sR^M$. Let $X$ be an analytic vector field on \( \sR^M\), and let \( \vtheta(t)\) denote the solution to the Cauchy problem \( \dot{\vtheta} = X(\vtheta),\vtheta(0) = \vtheta_0 \). 
We say that \( \mathcal{M} \) is an \textbf{invariant manifold} of  \( X \) if for every \( \vtheta_0 \in \mathcal{M} \), the solution \( \vtheta(t) \) remains in \( \mathcal{M} \) for all \( t \) in its maximal interval of existence.  We also say \textbf{$\fM$ is invariant under $X$}.
\label{def:invariant manifold}
\end{definition}

\section{Experiments}

\subsection{Details of Figure~\ref{fig:experiment}}
\label{sec:details of figure}
We consider  the model  $F(\vtheta)(\vx)=\sum_{i=1}^4 k_i(\vtheta_i)$
such that:
$$k_i(\vtheta_i)=2a_ib_i+(a_i^2+b_i^2+c_i^2)c_i \, , i=1,2$$ $$k_i(\vtheta)=a_i\tanh(a_i)-(e^{b_i}-1)^2\,, i=3,4.$$
One can readily verify that this model satisfies Definition~\ref{def:general 2-layer}, and $t_i^+=t_i^-=1$ for all $i\in [n]$. 
We have the data matrix
\[
\vX = \begin{pmatrix}
1 & 0.5 & 0.7 & 0 \\
0.5 & 1 & 0.1 & 0.7
\end{pmatrix}, \quad 
\vy = \begin{pmatrix}
1 \\
0
\end{pmatrix}.
\]

\textbf{Experiment details:} The parameters are initialized uniformly in the interval  $[-10^{-60},10^{-60}]$. The learning rate is $0.1$. We use \textit{mpmath} library to support calculation of  high-precision. In the code, we set mp.mp.dps = 200, which means that we use $200$ decimal places of precision for calculations.

\textbf{Output of Algorithm~\ref{alg:my_algorithm}:} In the following we calculate the output of Algorithm~\ref{alg:my_algorithm} under the input $\vX,\vy, t_i^+,t_i^-.$

\begin{itemize}
    \item $ p = 0$; The initialization is $\vk^{(0)}=\vs^{(0)}=\vzero$.
    \item $p=1$; We have $I_2=\{1,2,3,4\}.$
By calculation, $\vu=\vX^\T \vy=(1,0.5,0.7,0)^\T.$ So the time vector $\vdelta=(1,2,\frac{10}{7},+\infty).$ Therefore, the first neuron will be chosen as the feature. So $\vs^{(1)}=(1,0.5,0.7,0)^\T.$ $k_1^{(1)}$ is the solution of 
$$\min_{k_1\in \sR} \| \vX_{:,1}k_1-\vy\|_2^2.$$
This leads to $k_1^{(1)}=0.8.$
So $\vk^{(1)}=(0.8,0,0,0)^\T.$ To summarize, we have 
$$\vk^{(1)}=(0.8,0,0,0)^\T \, , \, \vs^{(1)}=(1,0.5,0.7,0)^\T.$$

\item $p=2$; We have $I_2=\{2,3,4\}.$ By calculation, we have 
$$u_2,u_3,u_4=-0.3,0.1,-0.28.$$ So we have 
$$\delta_2,\delta_3,\delta_4=5,3,\frac{25}{7}.$$
Therefore, $\delta_3$ is the smallest. So the third neuron is chosen. Then $\vs^{(2)}=(1,-0.4,1,-0.84)^\T.$ 

Besides, $k_1^{(2)}$ and  $k_3^{(2)}$ are the solution of 
$$\min_{(k_1,k_3)\in \sR^2} \| \vX_{:,1}k_1+\vX_{:,3}k_3-\vy\|_2^2 \quad  s.t. \quad  k_1\geq 0.$$
The solution is $k_1=0,k_3=1.4.$ So we have $\vk^{(2)}=(0,0,1.4,0)^\T.$  To summarize, we have 
$$\vk^{(2)}=(0,0,1.4,0)^\T \, , \, \vs^{(2)}=(1,-0.4,1,-0.84)^\T.$$

\item $p=3$; We have $I_2=\{1,2,4\}.$ By calculation, we have 
$$u_1,u_2,u_4=-0.05,-0.13,-0.098.$$ So we have 
$$\delta_1,\delta_2,\delta_4=20,\frac{60}{13}, \frac{80}{49}.$$
Therefore, $\delta_4$ is the smallest. So the fourth neuron is chosen. $k_3^{(3)}$ and  $k_4^{(3)}$ are the solution of 
$$\min_{(k_3,k_4)\in \sR^2} \| \vX_{:,3}k_3+\vX_{:,4}k_4-\vy\|_2^2 \quad  s.t. \quad  k_3\geq 0.$$
The solution is $k_3=\frac{10}{7},k_4=-\frac{10}{49}.$ So we have $\vk^{(3)}=(0,0,\frac{10}{7},-\frac{10}{49})^\T.$ Besides, it is readily to verify that $\vX\vk^{(3)}=\vy.$ So the iteration stops.

To summarize, Algorithm~\ref{alg:my_algorithm} ends at $p=3$, and 
we have 
$$\vk^{(3)}=(0,0,\frac{10}{7},-\frac{10}{49})^\T.$$

\end{itemize}

\textbf{Comparison of $\vk^{(p)}$ to the dashed line in Figure~\ref{fig:experiment}:} 
\begin{itemize}
    \item In the first dashed line, we have $(k_1,k_2,k_3,k_4)^\T=(0.800,8.34e-70,1.51e-28,-1.35e-100)^\T.$ This value is close to $\vk^{(1)}=(0.8,0,0,0)^\T.$
    \item In the second dashed line, we have $(k_1,k_2,k_3,k_4)^\T=(-1.01e-6,-6.23e-78,1.400,-1.34e-17)^\T.$ This value is close to $\vk^{(2)}=(0,0,1.4,0)^\T.$

    \item In the second dashed line, we have $(k_1,k_2,k_3,k_4)^\T=(-1.01e-6,-6.23e-78,1.400,-1.34e-17)^\T.$ This value is close to $\vk^{(2)}=(0,0,1.4,0)^\T.$

    \item In the third dashed line, we have $(k_1,k_2,k_3,k_4)^\T=(-1.71e-6,-2.48e-57,1.429,-0.204)^\T.$ This value is close to $\vk^{(3)}=(0,0,\frac{10}{7},-\frac{10}{49})^\T.$
\end{itemize}

\textbf{Difference between gradient flow and gradient descent:}
In gradient flow, as shown in Theorem~\ref{thm:SIM}, the parameter $\vtheta_i$ is strictly restricted to a simple curve. As a consequence, if $k_i(\vtheta_i)$ wants to change its sign, then  $\vtheta_i$ must 
must pass its initialization. Therefore, if $k_i(\vtheta_i)$ wants to change its sign, $|k_i|$ must decrease to its  initialization scale.

However, in Figure~\ref{fig:experiment}, we see that the value of $k_1$ is $0.800$ at the first dashed line, and $-1.01e-6$ at the second dashed line. Between the two dashed lines, the value of $|k_1|$ does not decrease to its initialization scale $1e-60.$ This is different from the dynamics of gradient flow. We attribute this phenomenon to the fact that Structural Invariant Manifold in Theorem~\ref{thm:SIM} is not invariant under gradient descent (since $O_{\fF_i}(\vtheta_i)$ is not a straight line). The failure of Theorem~\ref{thm:SIM} in gradient descent is also evidenced by~\citet{ziyin2024parameter}.

\subsection{Deep Diagonal Linear Network}

Figure~\ref{fig:combined_2}  illustrates the training dynamics of a Deep Diagonal Linear Network ($w = a \odot b \odot c$) optimized via gradient descent. The experiment utilizes a synthetic dataset with $M=3$ samples and $N=8$ features. We employ a small initialization scale ($\approx 10^{-4}$) while maintaining a strict layer-wise hierarchy ($a > b > c$) to facilitate theoretical analysis.

Figures~\ref{fig:loss_2}  and ~\ref{fig:dynamics_2}  demonstrate a clear synchronization between the descent of the loss curve and the sequential growth of neural parameters. The learning process exhibits distinct stage-wise behavior, where the system transitions between saddle points. We define the end of each stage as the point where the active parameters stabilize and the gradient of the residual becomes negligible.

Figure~\ref{fig:predictions_2}  presents the theoretical ``growth time" for each neuron at the onset of every stage, calculated by our algorithm as the integrated virtual distance required to escape the saddle point. Notably, these theoretical predictions align perfectly with the empirical emergence order and relative timing observed in Figure~\ref{fig:experiment}.

Figure~\ref{fig:combined} demonstrates the controllability of feature selection through initialization. While the network naturally prioritizes features based on data correlation, we show that this order can be manipulated. By selectively amplifying the initialization magnitude of the 5th neuron by a factor of 5 (while keeping others at the original scale), we successfully induce this neuron to emerge first, overriding the natural data-driven order. Crucially, our theoretical framework accurately captures this perturbation, correctly predicting the altered emergence sequence in Figure~\ref{fig:combined}.

Figure~\ref{fig:placeholder} investigates the implicit bias under a specific ``Zero-Balanced" initialization scheme ($a=b, c=0$, with scale $\approx 10^{-4}$). This setting creates a strict saddle point at initialization. To verify the regularization properties, we computed the mathematical minimum $\ell_1$-norm solution for the generated dataset. As shown in Figure~\ref{fig:combined}, as the loss approaches zero, the parameters converge precisely to the $\ell_1$ solution rather than the sparse ground truth ($\ell_0$ solution). This confirms that the $c=0$ initialization induces a strong inductive bias towards $\ell_1$ regularization, effectively driving the network to select the solution with the minimum norm among infinite possibilities.

\begin{figure}[h]
    \centering
    \begin{subfigure}[b]{0.38\linewidth}
        \centering
        \includegraphics[height=5cm, keepaspectratio]{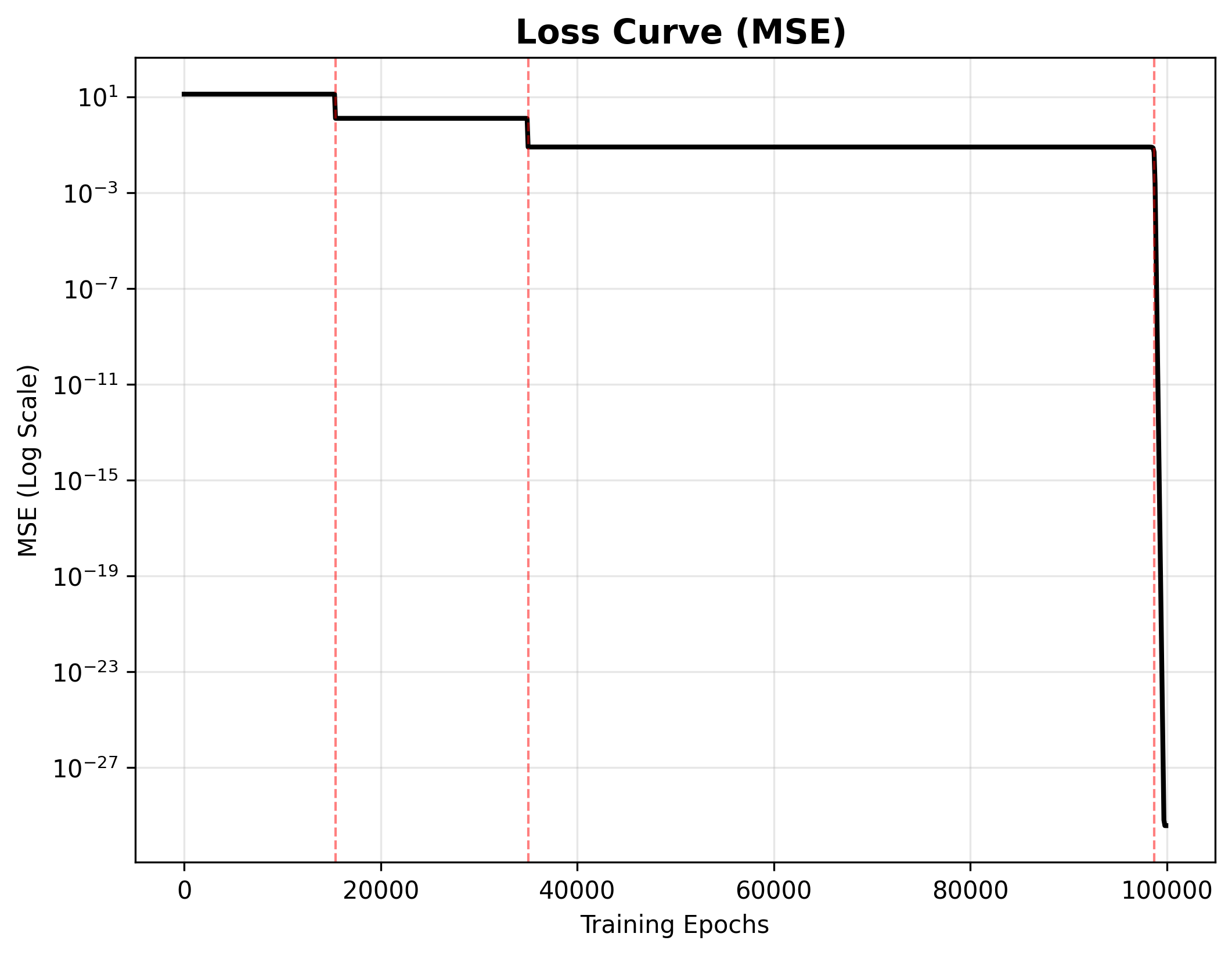}
        \caption{Training loss of the Deep Diagonal Linear Network } 
        \label{fig:loss_2}
    \end{subfigure}
    \hfill
    \begin{subfigure}[b]{0.6\linewidth}
        \centering
        \includegraphics[height=5cm, keepaspectratio]{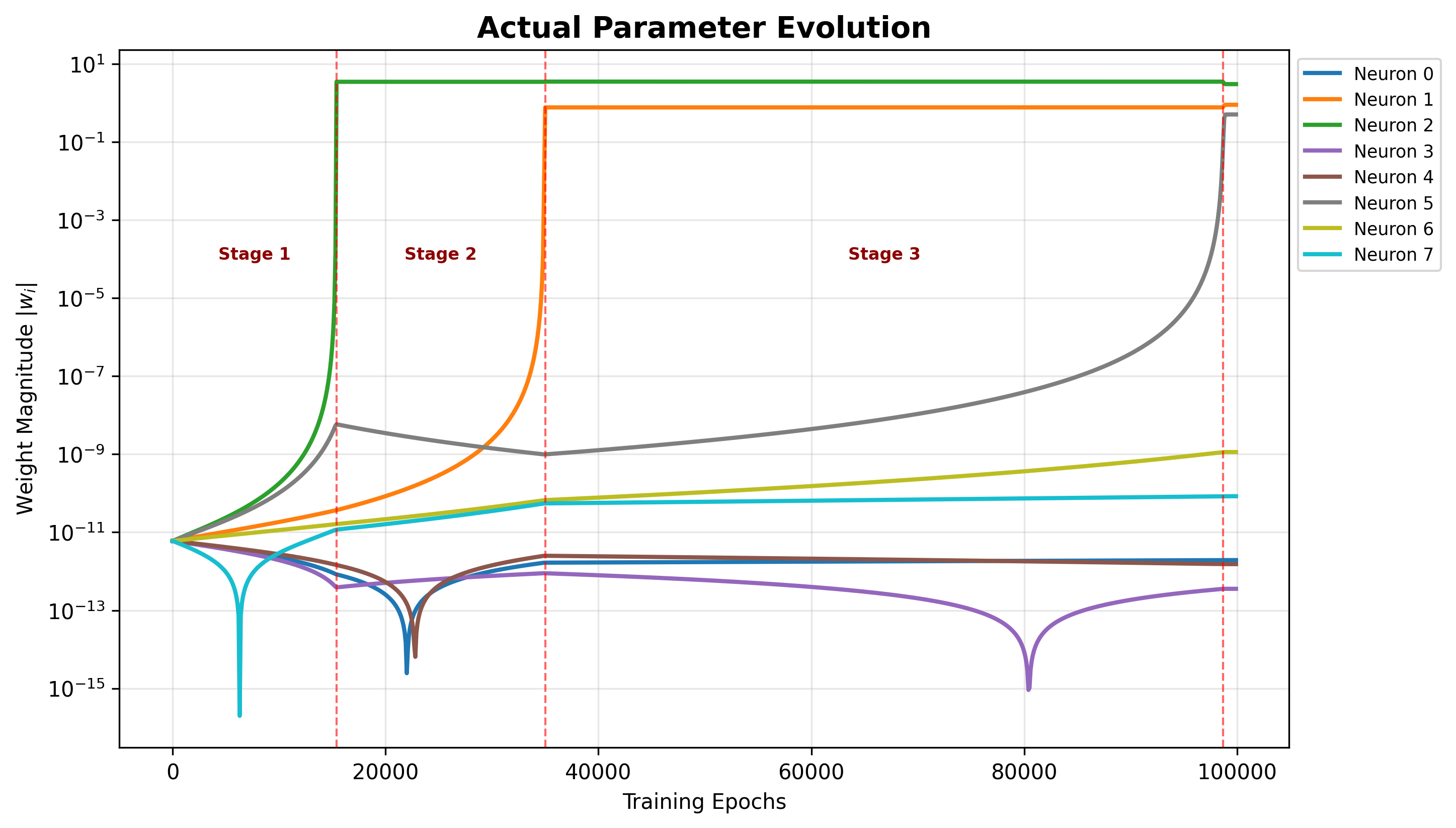}
        \caption{Parameter trajectories of individual neurons across training epochs}
        \label{fig:dynamics_2}
    \end{subfigure}

    \vspace{1em}

    \begin{subfigure}[b]{\linewidth}
        \centering
        \includegraphics[width=\linewidth]{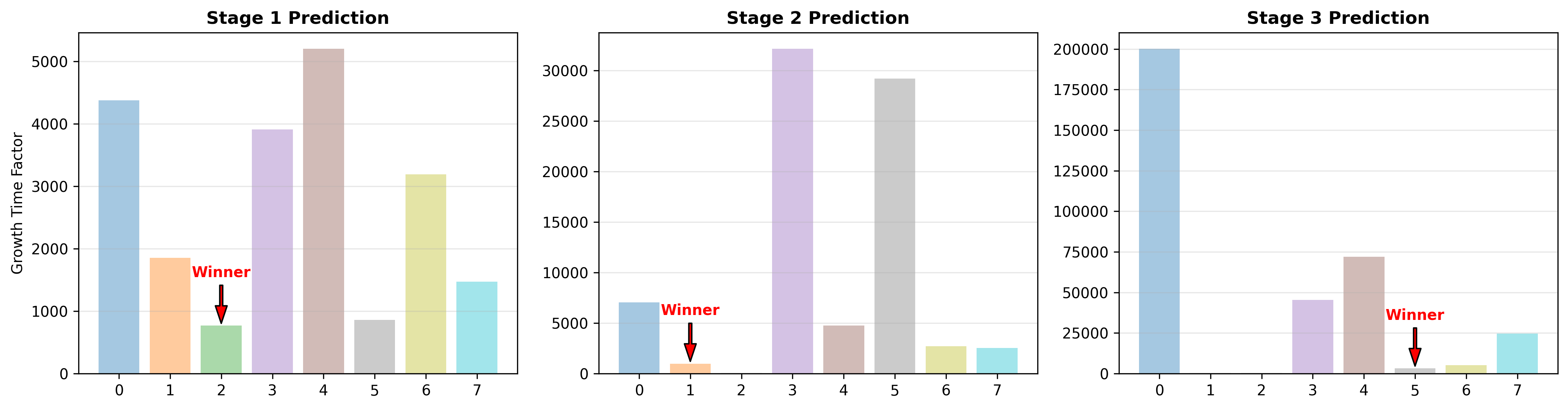}
        \caption{Theoretical prediction of neuron growth times at each stage}
        \label{fig:predictions_2}
    \end{subfigure}

    \caption{Training dynamics of the Deep Diagonal Linear Network with uniform initialization scale.
(a) The training loss curve. The vertical red dashed lines correspond to the same epochs marked in (b), indicating stage transitions.
(b) Evolution of the absolute parameter magnitudes $|w_i| = |a_i b_i c_i|$ for the 8 neurons throughout training.
(c) Theoretical predictions of the growth time calculated at the beginning of each stage. The algorithm identifies the neuron with the shortest predicted growth time as the "winner" for the subsequent stage.
    }
    \label{fig:combined_2}
\end{figure}

\begin{figure}[htbp]
    \centering
    \begin{subfigure}[b]{0.38\linewidth}
        \centering
        \includegraphics[height=5cm, keepaspectratio]{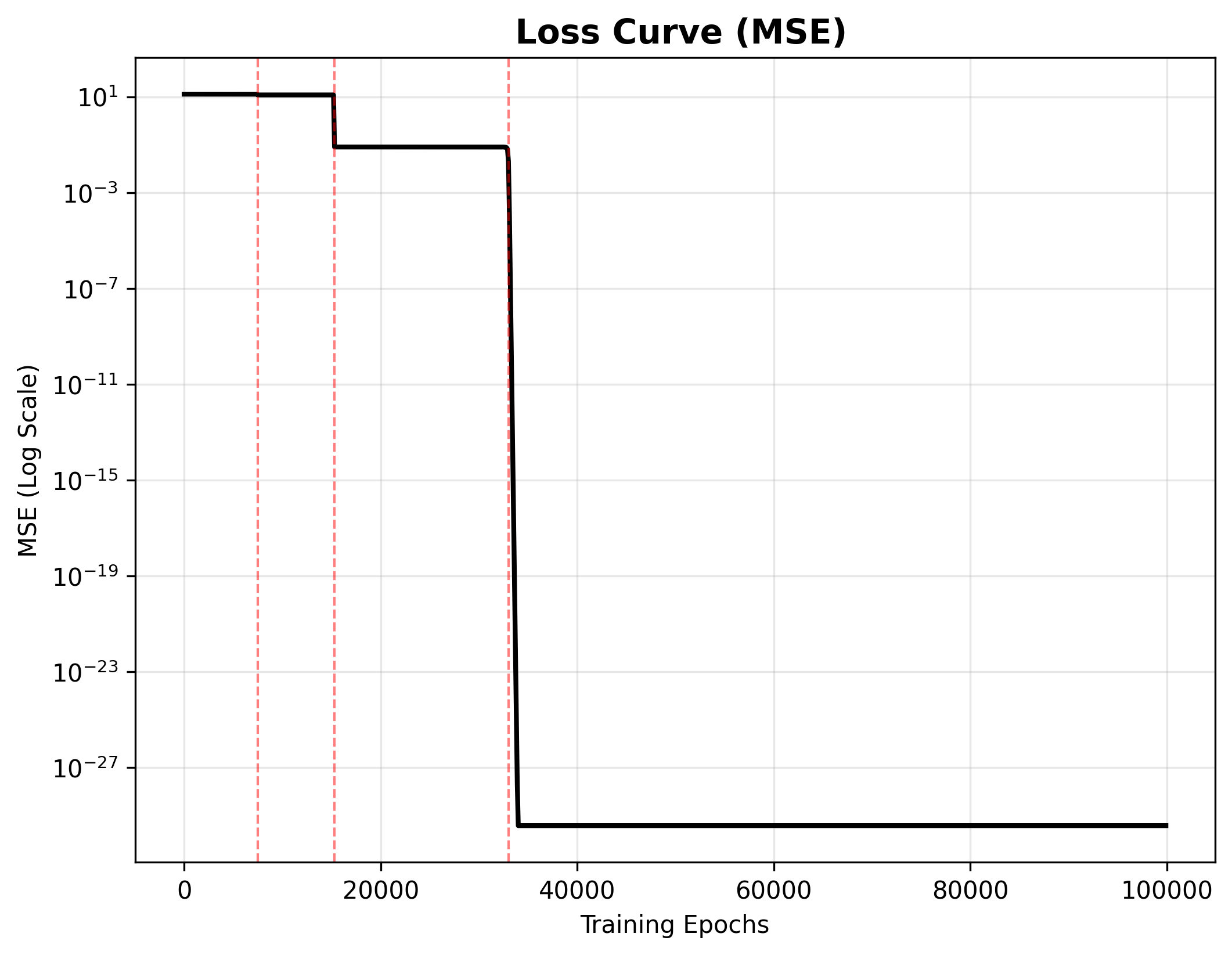}
        \caption{Training loss curve of the Deep Diagonal network }
        \label{fig:loss_3}
    \end{subfigure}
    \hfill
    \begin{subfigure}[b]{0.6\linewidth}
        \centering
        \includegraphics[height=5cm, keepaspectratio]{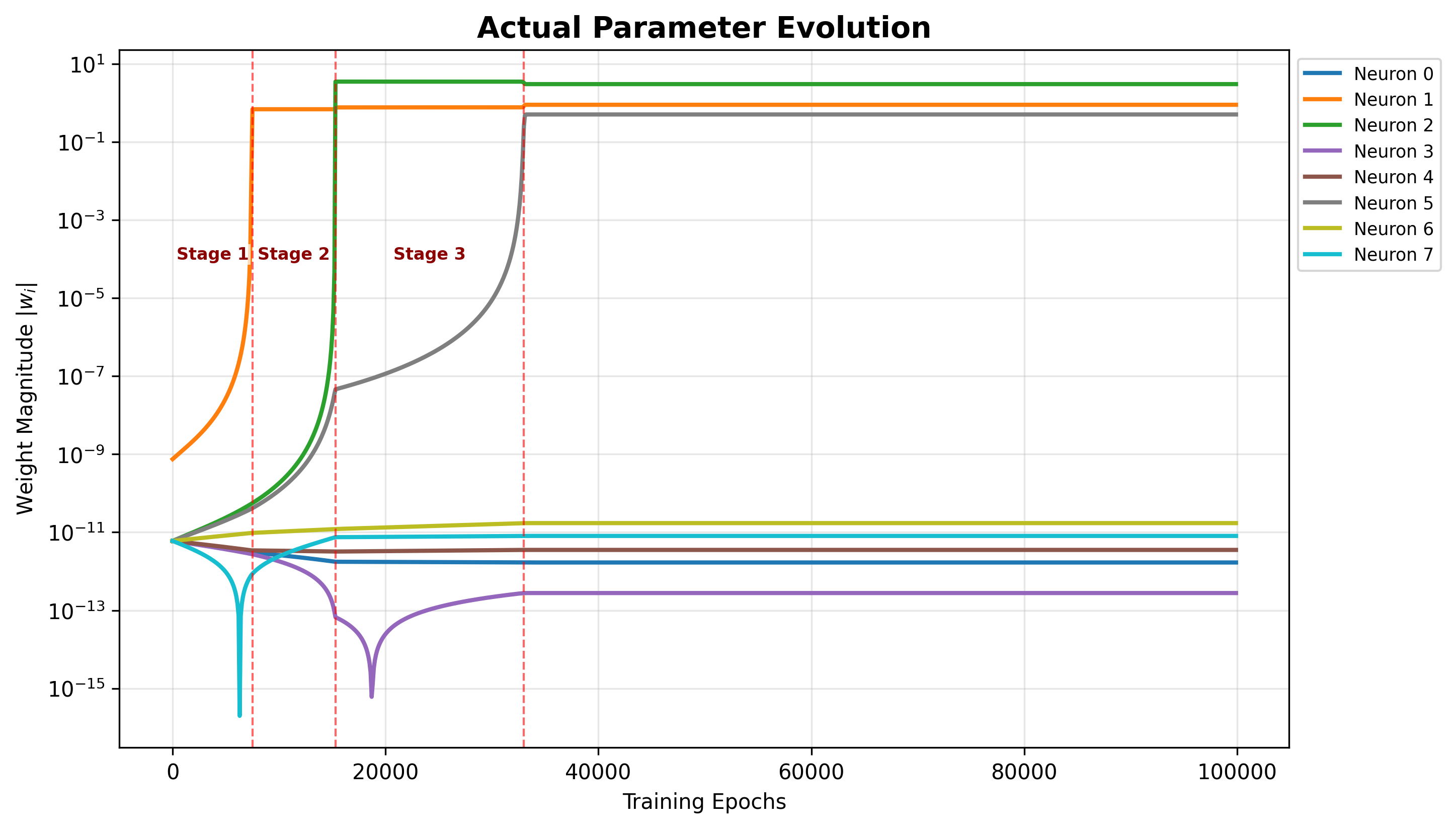}
        \caption{Parameter trajectories of individual neurons across training epochs}
        \label{fig:dynamics}
    \end{subfigure}

    \vspace{1em}

    \begin{subfigure}[b]{\linewidth}
        \centering
        \includegraphics[width=\linewidth]{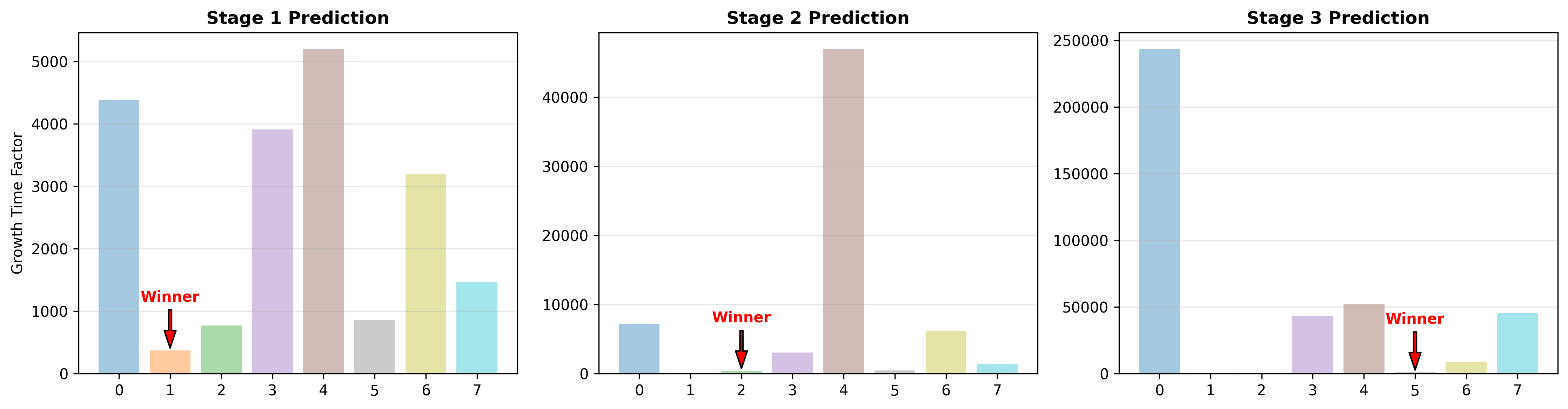}
        \caption{Theoretical prediction of neuron growth times at each stage}
        \label{fig:predictions}
    \end{subfigure}

    \caption{Training dynamics of the Deep Diagonal Linear Network with biased initialization. The initialization scale of the first neuron is amplified by a factor of 5, while others remain unchanged.
(a) The training loss curve. The vertical red dashed lines correspond to the same epochs marked in (b).
(b) Evolution of the absolute parameter magnitudes $|w_i|$ throughout training. Note that the first neuron emerges earliest due to its larger initialization scale.
(c) Theoretical predictions of the growth time. The algorithm correctly identifies the first neuron as the "winner" of the initial stage, capturing the effect of the initialization bias.
    }
    \label{fig:combined}
\end{figure}

\begin{figure}
    \centering
    \includegraphics[width=1.0\linewidth]{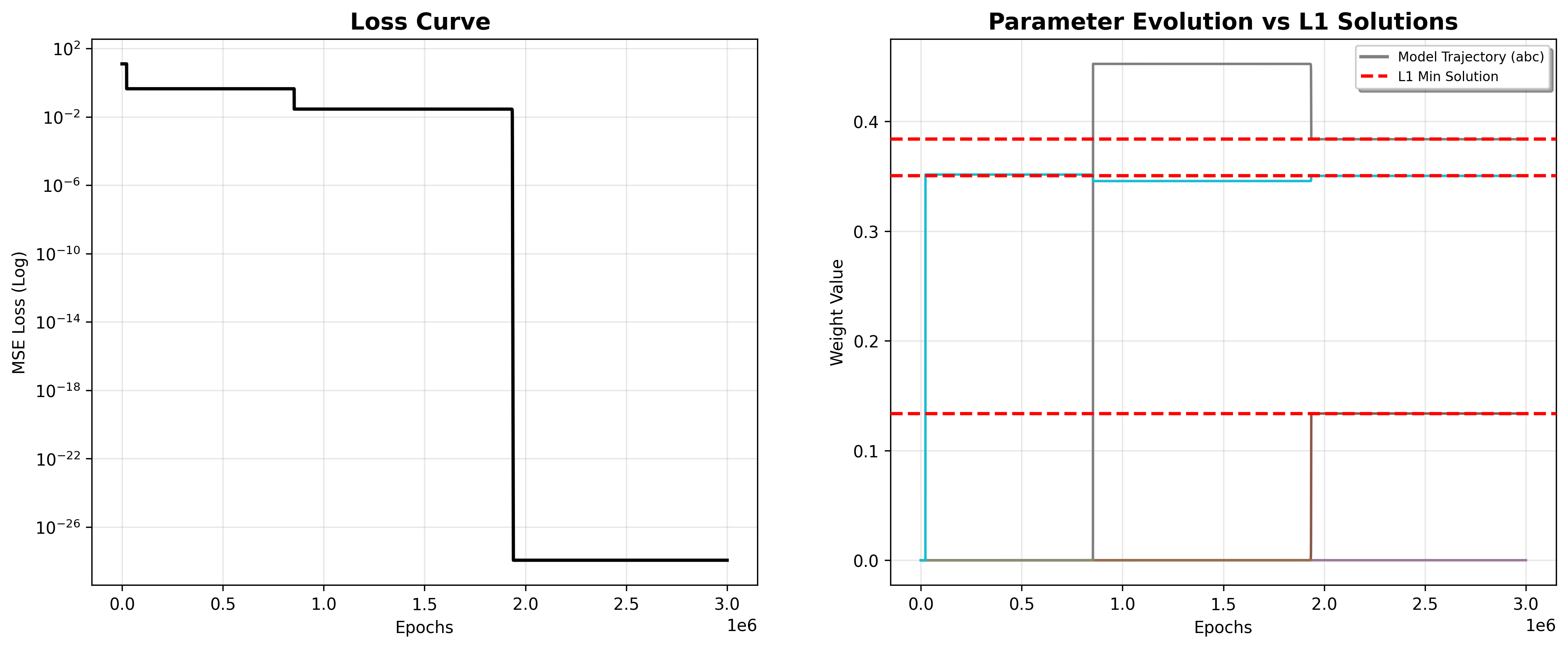}
    \caption{Verification of implicit $\ell_1$ regularization under "Zero-Balanced" initialization ($a=b, c=0$).
The left figure is Training loss curve. The right figure
depicts parameter evolution trajectories compared against the theoretical baseline. The red dashed lines represent the mathematical minimum $\ell_1$-norm solution for the dataset. }
    \label{fig:placeholder}
\end{figure}

\end{document}